\def\year{2019}\relax

\documentclass[letterpaper]{article} 
\usepackage{aaai19}  
\usepackage{times}  
\usepackage{helvet}  
\usepackage{courier}  
\usepackage{url}  
\usepackage{graphicx}  
\usepackage{amsthm}
\usepackage{amsfonts,amssymb}
\usepackage{amsmath}
\newtheorem{theorem}{Theorem}
\newtheorem{lemma}{Lemma}

\newtheorem{assumption}{Assumption}
\newtheorem{remark}{Remark}
\newtheorem{proposition}{Proposition}
\usepackage{bm}
\usepackage{color}
\usepackage{algorithm}
\usepackage{algorithmic}
\usepackage{graphicx}
\usepackage{subfigure}
\usepackage{sidecap}
\usepackage{paralist}
\usepackage{booktabs}

\def\UrlFont{\rm}  
\usepackage{graphicx}  
\title{Expected Sarsa($\lambda$) with Control Variate for Variance Reduction }
 
\begin{document}
\author{Long Yang, Yu Zhang, Jun Wen, Qian Zheng, Pengfei Li, Gang Pan
 \\
  Department of Computer Science,
  Zhejiang University\\
  \texttt{\{yanglong,hzzhangyu,junwen,qianzheng,pfl,gpan\}@zju.edu.cn} 
}
\maketitle

\begin{abstract}
Off-policy learning is powerful for reinforcement learning.
However, the high variance of off-policy evaluation is a critical challenge, which causes off-policy learning falls into an uncontrolled instability.
In this paper, for reducing the variance, we introduce control variate technique to $\mathtt{Expected}$ $\mathtt{Sarsa}$($\lambda$) and propose a tabular $\mathtt{ES}$($\lambda$)-$\mathtt{CV}$ algorithm.
We prove that if a proper estimator of value function reaches, the proposed $\mathtt{ES}$($\lambda$)-$\mathtt{CV}$ enjoys a lower variance than $\mathtt{Expected}$ $\mathtt{Sarsa}$($\lambda$).
Furthermore, to extend $\mathtt{ES}$($\lambda$)-$\mathtt{CV}$ to be a convergent algorithm with linear function approximation,
we propose the $\mathtt{GES}$($\lambda$) algorithm under the convex-concave saddle-point formulation.
We prove that the convergence rate of $\mathtt{GES}$($\lambda$) achieves $\mathcal{O}(1/T)$, which matches or outperforms lots of state-of-art gradient-based algorithms, but we use a more relaxed condition.
Numerical experiments show that the proposed algorithm performs better with lower variance than several state-of-art gradient-based TD learning algorithms: $\mathtt{GQ}$($\lambda$), $\mathtt{GTB}$($\lambda$) and $\mathtt{ABQ}$($\zeta$).
\end{abstract}

\section{Introduction}

Off-policy learning is powerful for reinforcement learning due to it learns the target policy from the data generated by another policy~\cite{sutton1998reinforcement}.
However, suffering high variance is a critical challenge for off-policy learning \cite{Tamar2016learning}, 
which roots in the discrepancy of distribution between target policy and behavior policy.
The resources of high variance of off-policy learning can be divided into two parts, \textbf{(I)} one is tabular case which has to do with the target of the update, \textbf{(II)} one is with function approximation which has to do with the distribution of the update \cite{sutton2018reinforcement}.

In this paper, we mainly focus on the variance reduce technique to an important off-policy algorithm: $\mathtt{Expected}$ $\mathtt{Sarsa}$($\lambda$).
We introduce control variate to $\mathtt{Expected}$ $\mathtt{Sarsa}$($\lambda$) and propose $\mathtt{Expected}$ $\mathtt{Sarsa}$($\lambda$) with control variate ($\mathtt{ES}$($\lambda$)-$\mathtt{CV}$) for the tabular case.
The control variate method is one of the most effective variance reduction techniques in statistical inference~\cite{rubinstein2016simulation}.
Control variate is an additional term that has zero expectation, which implies introducing control variate does not change the expectation of update.
Thus, learning with control variate does not introduce any biases, but it is potential to enjoy much lower variance~\cite{thomas2016data,de2018per,liu2018ACT}.
Sutton and Barto \shortcite{sutton2018reinforcement} (section 12.9) firstly introduces control variate to $\mathtt{Expected}$ $\mathtt{Sarsa}$($\lambda$), but their analysis is limited in linear function approximation. 
Later, De Asis and Sutton~\shortcite{de2018per} further introduce control variate to multi-step TD learning, but it constrains on off-line learning (which is extremely expensive for training).

Despite being easy to implement, competitive to the state of the art methods, and being used in practice, in RL, the TD learning with control variate technique lacks a robust theoretical analysis.
In this paper, we focus on the theoretical analysis of $\mathtt{ES}$($\lambda$)-$\mathtt{CV}$.
We prove that the tabular $\mathtt{ES}$($\lambda$)-$\mathtt{CV}$ converges at an exponential fast for off-policy evaluation without biases.
Furthermore, we analyze all the random sources lead to the variance of $\mathtt{ES}$($\lambda$)-$\mathtt{CV}$, and
we prove that if a proper estimator of value function reaches, $\mathtt{ES}$($\lambda$)-$\mathtt{CV}$ enjoys a lower variance than $\mathtt{Expected}$ $\mathtt{Sarsa}$($\lambda$).

Furthermore, we show the variance reduction way presented by~\cite{sutton2018reinforcement} (section 12.9) to extend $\mathtt{ES}$($\lambda$)-$\mathtt{CV}$ with function approximation is unstable.
Although this instability has been realized by Sutton and Barto \shortcite{sutton2018reinforcement},
it is only an intuitive guess inspired previous works \cite{maei2011gradient,mahmood2017_a_incremental}.
In this paper,
we provide a simple but rigorous theoretical analysis to illustrate the instability appears in ~\cite{sutton2018reinforcement}. 
We also demonstrate this instability by a typical example.

To extend the $\mathtt{ES}$($\lambda$)-$\mathtt{CV}$ with function approximation be a convergent and stable algorithm, we propose $\mathtt{GES}$($\lambda$) algorithm under the 
the convex-concave saddle-point formulation~\cite{liu2015finite}.
We prove the convergence rate of  $\mathtt{GES}$($\lambda$) achieves $\mathcal{O}(1/T)$, where $T$ is the number of iterations. 
Our $\mathcal{O}(1/T)$ matches or outperforms extensive state-of-art works~\cite{Nathaniel2015ontd,liu2015finite,wang2017finite,gal2018finite,gal2018finitesample,touati2018convergent}, 
with a more relaxed condition than theirs.
Besides, we prove the results of convergence rate without the assumption that the objective is strongly convex in the primal space and strongly concave in the dual space~\cite{Balamurugan2016Stochastic}.

Finally, we conduct numerical experiments to show that the proposed algorithm is stable and converges faster with lower variance than lots of state-of-art  gradient-based TD learning algorithms: $\mathtt{GQ}$($\lambda$) \cite{maei2010GQ}, $\mathtt{GTB}$ ($\lambda$) \cite{touati2018convergent}, and $\mathtt{ABQ}$ ($\zeta$) \cite{Mahmood2017_b_multi}.

\subsection{Contributions}
 \begin{itemize}
 \item 
We introduce control variate technique to $\mathtt{Expected}$ $\mathtt{Sarsa}$($\lambda$) and propose a tabular $\mathtt{ES}$($\lambda$)-$\mathtt{CV}$ algorithm.
We prove that if a proper estimator of value function reaches, the proposed $\mathtt{ES}$($\lambda$)-$\mathtt{CV}$ enjoys a lower variance than $\mathtt{Expected}$ $\mathtt{Sarsa}$($\lambda$).

  \item 
We propose the $\mathtt{GES}$($\lambda$), which extends $\mathtt{ES}$($\lambda$)-$\mathtt{CV}$ to be a convergent algorithm with linear function approximation.
We prove that the convergence rate of $\mathtt{GES}$($\lambda$) achieves $\mathcal{O}(1/T)$, which matches or outperforms lots of state-of-art gradient-based algorithms, but we use a more relaxed condition.
 
  \end{itemize}
\section{Preliminary and Some Notations}
 In this section, we introduce some necessary notations about reinforcement learning, temporal difference learning and $\lambda$-return.
 For the limitation of space, we more discussions about $\lambda$-return in Appendix A and B.

\textbf{Reinforcement Learning}
The reinforcement learning (RL) is often formalized as \emph{Markov decision processes} (MDP)~\cite{sutton1998reinforcement} which considers 5-tuples form $\mathcal{M}=(\mathcal{S},\mathcal{A},\mathcal{P},\mathcal{R},\gamma)$. $\mathcal{S}$ is the set contains all states, $\mathcal{A}$ is the set contains all actions.
$\mathcal{P} : \mathcal{S}\times\mathcal{A}\times\mathcal{S}\rightarrow[0,1]$,
$P_{s s^{'}}^a=\mathcal{P}(S_{t}=s^{'}|S_{t-1}=s,A_{t-1}=a)$ is the probability for the state transition from $s$ to $s^{'}$ under taking the action $a$.
$\mathcal{R} : \mathcal{S}\times\mathcal{A}\rightarrow\mathbb{R}^{1}$, $\mathcal{R}_{s}^{a}=\mathbb{E}[R_{t+1}|S_{t}=s,A_{t}=a]$. 
$\gamma\in(0,1)$ is the discount factor. 

A \emph{policy} is a probability distribution on $\mathcal{S}\times\mathcal{A}$. 
\emph{Target policy} $\pi$ is the policy will be learned and 
\emph{behavior policy} $\mu$ is used to generate behavior. 
$\tau=\{S_{t}, A_{t}, R_{t+1}\}_{t\ge 0}$ denotes a \emph{trajectory}, 
where $A_{t}\sim\mu(\cdot|S_{t})$ and $S_{t+1}\sim \mathcal{P}(\cdot|S_{t},A_{t})$. 
For a given policy $\pi$, its \emph{state-action value function} $ q^{\pi}(s,a) = \mathbb{E}_{\pi}[G_{t}|S_{t} = s,A_{t}=a]
$, \emph{state value function} $v^{\pi}(s) = \mathbb{E}_{\pi}[G_{t}|S_{t} = s]$,
where $G_{t}=\sum_{k=0}^{\infty}\gamma^{k}R_{k+t+1}$ and $\mathbb{E}_{\pi}[\cdot|\cdot]$ denotes an conditional  expectation on all actions which be selected according to $\pi$.
It is known that $q^{\pi}(s,a)$ is the unique fixed point ~\cite{bertsekas2012dynamic} of \emph{Bellman operator} $\mathcal{B}^{\pi}$,
\begin{flalign}
\label{bellman-equation}
\mathcal{B}^{\pi} q^{\pi}=q^{\pi},
\end{flalign}
which is known as \emph{Bellman equation}, 
where \[\mathcal{B}^{\pi}:  q\mapsto R+\gamma P^{\pi}q,\]
 $P^{\pi}$$\in\mathbb{R}^{|\mathcal{S}| \times |\mathcal{S}|}$ and $R$$\in\mathbb{R}^{|\mathcal{S}|\times|\mathcal{A}|}$, the corresponding elements of $P^{\pi}$ and $R$ are:
 \[
 P^{\pi}_{ss^{'}}= \sum_{a \in \mathcal{A}}\pi(a|s)P^{a}_{ss^{'}},R(s,a)=\mathcal{R}_{s}^{a}.
 \]

\textbf{TD Learning}~
Temporal difference (TD) learning 
\cite{sutton1988learning} 
is one of the most important methods to solve model-free RL (in which, we cannot get $\mathcal{P}$).
For the trajectory $\tau$, TD learning is defined as, $\forall ~t\ge0$
\begin{flalign}
\label{td-learning}
Q(S_{t},A_{t})\leftarrow Q(S_{t},A_{t})+\alpha_t\delta_{t},
\end{flalign}
where $Q(\cdot,\cdot)$ is an estimate of $q^{\pi}$, $\alpha_t$ is step-size and $\delta_{t}$ is TD error. 
Let $Q_{t}\overset{\text{def}}=Q(S_{t},A_{t})$,
if $\delta_{t}$ is \[\delta_{t}^{\text{S}}\overset{\text{def}}=R_{t+1}+\gamma Q_{t+1} - Q_{t},\]
above update (\ref{td-learning}) is $\mathtt{Sarsa}$ algorithm~\cite{rummery1994line}.
If $\delta_{t}$ is 
\begin{flalign}
\label{def:es-delta}
\delta_{t}^{\text{ES}}=R_{t+1}+\mathbb{E}_{\pi}[Q(S_{t+1},\cdot)]- Q_{t},
\end{flalign}
update (\ref{td-learning}) is $\mathtt{Expected~Sarsa}$~\cite{van2009theoretical}, where
$\mathbb{E}_{\pi}[Q(S_{t+1},\cdot)]=\sum_{a\in\mathcal{A}}\pi(a|S_{t+1})Q(S_{t+1},a)$. 
If $\pi$ is reduced to greedy policy, then $\mathtt{Expected~Sarsa}$ reduces to $\mathtt{Q\text{-}learning}$~\cite{watkins1989learning}.

\textbf{Expected Sarsa$(\lambda)$}~
The standard \emph{forward view} of $\lambda$-return \cite{sutton1998reinforcement} of on-policy $\mathtt{Expected~Sarsa}$ is defined as follows,
\begin{flalign}
\label{def:ES}
G_{t}^{\lambda,\text{ES}}=(1-\lambda)\sum_{n=1}^{\infty}\lambda^{n-1}G_{t}^{t+n},
\end{flalign}
where $G_{t}^{t+n}=\sum_{i=0}^{n-1}\gamma^{i}R_{t+i+1}+\gamma^{n}\bar{Q}_{t+n}$ 
is $n$-\emph{step return} of $\mathtt{Expected~Sarsa}$, and $\bar{Q}_{t+n}=\mathbb{E}_{\pi}[Q(S_{t+n},\cdot)]$.
We can write $G_{t}^{\lambda,\text{ES}}$ recursively as follows (the detail is provided in Appendix A),
\begin{flalign}
\label{on-plicy-ES-lambda-recursive}
G_{t}^{\lambda,\text{ES}}=R_{t+1}+\gamma[(1-\lambda)\bar{Q}_{t+1}+\lambda G_{t+1}^{\lambda,\text{ES}}].
\end{flalign}
Now, we introduce an unbiased
\footnote{
    How to define the $\lambda$-return of $\mathtt{Expected~Sarsa}$ for off-policy learning? 
    Can we follow the way of (\ref{def:ES}) straightforwardly? 
    Unfortunately, for the off-policy, the above idea cannot converge to $q^{\pi}$. 
    In fact, $n$-step return of $\mathtt{Expected~Sarsa}$ is sampled according to \[R_{t:t+n}=\sum_{t=0}^{n}\gamma^{t}(P^{\mu})^{t}R_{t+1}+\gamma^{n+1}(P^{\mu})^{n}P^{\pi}Q.\]
    Then according to (\ref{def:ES}), we define the $\lambda$-return of $\mathtt{Expected~Sarsa}$ as follows, 
    \[
    (1-\lambda)\sum_{n=0}^{\infty}\lambda^{n}
    R_{t:t+n}
    =((1-\lambda)\mathcal{B}^{\pi}+\lambda\mathcal{B}^{\mu})Q,
    \] which converges to $(1-\lambda)q^{\pi}+\lambda q^{\mu}\ne q^{\pi}$. 
    This is 
    the fixed point of $(1-\lambda)\mathcal{B}^{\pi}+\lambda\mathcal{B}^{\mu}\ne\mathcal{B}^{\pi}$ and it is a biased estimate of $q^{\pi}$.
}
recursive $\lambda$-return of $\mathtt{Expected~Sarsa}$ for off-policy learning,
\begin{flalign}
\label{off-es-recursive}
G_{t}^{\lambda\rho,\text{ES}}=R_{t+1}+\gamma[(1-\lambda)\bar{Q}_{t+1}+\lambda\rho_{t+1}G_{t+1}^{\lambda\rho,\text{ES}}],
\end{flalign}
where $\rho_{t+1}=\pi(A_{t+1}|S_{t+1})/\mu(A_{t+1}|S_{t+1})$ is importance sampling. 
Eq.(\ref{off-es-recursive}) firstly appears in~\cite{maei2010GQ,maei2011gradient}, but in which it is limited in function approximation. We develop (\ref{off-es-recursive}) to be a general version which is conducive to the theoretical analysis of the following paragraph.
The following Proposition \ref{prop2} illustrates that $G_{t}^{\lambda\rho,\text{ES}}$ (\ref{off-es-recursive}) is an unbiased estimate of $q^{\pi}$.
\begin{proposition}
    \label{prop2}
    Let $\mu$ and $\pi$ be the behavior and target policy, respectively. 
    For the $\lambda$-return (\ref{off-es-recursive}), 
    we have
    \[\mathbb{E}_{\mu}[G_{t}^{\lambda\rho,\emph{ES}}|(S_{t},A_{t})=(s,a)]=q^{\pi}(s,a).\]
\end{proposition}
For the limitation of space, more discussions about $\lambda$-return of Sarsa, Eq.(\ref{on-plicy-ES-lambda-recursive})-(\ref{off-es-recursive}), and the proof of Proposition \ref{prop2} are provided in Appendix A and B.
\section{{Expected~Sarsa}($\lambda$) with Control Variate}

In this section, we firstly define $\mathtt{Expected~Sarsa}$($\lambda$) with control variate (we use $\mathtt{ES}$($\lambda$)-$\mathtt{CV}$ for short).
Then, prove its linear convergence rate of  $\mathtt{ES}$($\lambda$)-$\mathtt{CV}$ for policy evaluation.
Finally, we analyze the variance of $\mathtt{ES}$($\lambda$)-$\mathtt{CV}$.

\subsection{ES($\lambda$)-CV Algorithm }

We define $\mathtt{Expected~Sarsa}$($\lambda$) with control variate $\widetilde{G}_{t}^{\lambda\rho,\text{ES}}$ as follows
\begin{flalign}
\nonumber
\widetilde{G}_{t}^{\lambda\rho,\text{ES}}&=R_{t+1}+\gamma\Big[
(1-\lambda)\bar{Q}_{t+1}+
\lambda(\rho_{t+1}\widetilde{G}_{t+1}^{\lambda\rho,\text{ES}}
\\
\label{es-recursive-cv}
&+\underbrace{{\bar{Q}_{t+1}-\rho_{t+1}Q_{t+1}}}_{\text{control variate}})
\Big],
\end{flalign}
where the additional term $\bar{Q}_{t+1}-\rho_{t+1}Q_{t+1}$ is called control variate (CV).
The following fact \[\mathbb{E}_{\mu}[\bar{Q}_{t+1}-\rho_{t+1}Q_{t+1}]=0\]
implies that $\widetilde{G}_{t}^{\lambda\rho,\text{ES}}$ (\ref{es-recursive-cv}) extends $G_{t}^{\lambda\rho,\text{ES}}$ (\ref{off-es-recursive}) without introducing biases.
\begin{theorem}[Forward View of $\mathtt{ES}$($\lambda$)-$\mathtt{CV}$]
    \label{ES-Sarsa-CV-bias-variance}
    Let $\rho_{t:k}=\prod_{i=t}^{k}\rho_{i}$ denote the
    cumulated importance sampling from time $t$ to $k$, and we use $\rho_{t+1:t}=1$ for convention.
    The recursive  $\lambda$-return in Eq.(\ref{es-recursive-cv}) is equivalent to the following forward view: let $\delta_{l}^{\emph{ES}}$ be the TD error defined in (\ref{def:es-delta}), $G_{t}^{t}=Q_{t}$, $G_{t}^{t+n}= R_{t+1} +\gamma(\rho_{t+1}G_{t+1}^{t+n}+\bar{Q}_{t+1}-\rho_{t+1}Q_{t+1})$
    \begin{flalign}
\nonumber
    \widetilde{G}_{t}^{\lambda\rho,\emph{ES}}&=(1-\lambda)\sum_{n=1}^{\infty}\lambda^{n-1}G_{t}^{t+n}\\
        \label{ES-Gtt+n}
    &=Q_{t}+\sum_{l=t}^{\infty}(\gamma\lambda)^{l-t}\delta_{l}^{\emph{ES}}\rho_{t+1:l}.
    \end{flalign}
\end{theorem}
\begin{proof}
See Appendix C.
\end{proof}
\begin{remark}
Eq.(\ref{ES-Gtt+n}) illustrates that
for a given finite horizon trajectory $\{S_{t},A_{t},R_{t+1}\}_{t=0}^{h}$, 
the total update (\ref{es-recursive-cv}) reaches
 \begin{flalign}
 \sum_{t=0}^{h}(\gamma\lambda)^{t}\delta_{t}^{\text{ES}}\rho_{1:t},
 \end{flalign}
 which is off-line update of $\mathtt{ES}$($\lambda$)-$\mathtt{CV}$.
\end{remark}

\subsection{Policy Evaluation}
For policy evaluation, our goal is to estimate $q^{\pi}$ according to the trajectory collection $\mathcal{T}=\{\tau_k\}_{k\in\mathbb{N}}$, where 
$\tau_k=\{S_{t},A_{t},R_{t+1}\}_{t\ge0}\sim\mu$, $S_{t},A_{t}$, and
$R_{t+1}$ are dependent on the index $k$ strictly, and we omit coefficient $k$ to tight the expression without ambiguity.

The following $\lambda$-operator $\mathcal{B}^{\pi}_{\lambda} $ is a high level view of $\mathtt{ES}$($\lambda$)-$\mathtt{CV}$ (\ref{ES-Gtt+n}), and it
is helpful for us to introduce policy evaluation algorithm. $\forall~q\in\mathbb{R}^{|\mathcal{S}|\times |\mathcal{A}|},t\ge0$ 
\begin{flalign}
\label{operator-es-1}
\mathcal{B}^{\pi}_{\lambda} q&\overset{\text{def}}\mapsto q+\mathbb{E}_{\mu}[\sum_{l=t}^{\infty}(\lambda\gamma)^{l-t}\delta^{\text{ES}}_{l}\rho_{t+1:l}]\\
\label{operator-es}
&\overset{\text{(a)}}=q+(I-\lambda\gamma P^{\pi})^{-1}(\mathcal{B}^{\pi}q-q),
\end{flalign}
where $\mathcal{B}^{\pi}$ is defined in Eq.(\ref{bellman-equation}).
We provide the equivalence (a) in Appendix D.

\begin{theorem}[Policy Evaluation]
    \label{theorem-ope}
    For any initial $Q_{0}$, consider the trajectory $\mathcal{T}$ generated by $\mu$, and the following $Q_{k}$ is generated according to the $k$-th trajectory $\tau_{k}\in\mathcal{T}$, $k\ge1$,
    \begin{flalign}
    \label{PO}
     Q_{k+1}=\mathcal{B}^{\pi}_{\lambda}Q_{k}.
     \end{flalign}
     By iterating over $k$ trajectories, the upper-error of policy evaluation is bounded by 
     \begin{flalign}
     \label{error-bound-es-cv}
     \|Q_{k}-q^{\pi}\|\leq\big(\frac{\gamma-\lambda\gamma}{1-\lambda\gamma}\big)^{k}\|Q_0-q^{\pi}\|.
     \end{flalign}
\end{theorem}
\begin{proof}
See Appendix E.
\end{proof}
\begin{remark}
The forward view (off-line update) of $\mathtt{ES}$($\lambda$)-$\mathtt{CV}$ (\ref{ES-Gtt+n}) can be seen as sampled according to $Q_{t+1}=\mathcal{B}^{\pi}_{\lambda}Q_{t}$.
For any $\gamma\in(0,1),\lambda\in[0,1]$, then $\frac{\gamma-\lambda\gamma}{1-\lambda\gamma}\in(0,1)$,
thus Eq.(\ref{error-bound-es-cv}) implies (\ref{ES-Gtt+n}) converges to $q^{\pi}$ at a linear convergence rate.\end{remark}

\subsection{Variance Analysis}

\begin{theorem}[Variance Analysis of $\mathtt{ES}$($\lambda$)-$\mathtt{CV}$]
    \label{Variance-Analysis}
    Consider a single trajectory $\tau_{k}$ with ffinite horizon $H+1$, let $S_{t}=s,A_{t}=a,S_{t+1}=s^{'},A_{t+1}=a^{'}$,  $\mathbb{V}{\emph{ar}}\big[\widetilde{G}_{H+1}^{\lambda\rho,\emph{ES}}\big]=0$.
    The variance of $\widetilde{G}_{t}^{\lambda\rho,\emph{ES}}$ is given recursively as follows, 
    \begin{flalign}
        \nonumber
        \mathbb{V}{\emph{ar}}\big[\widetilde{G}_{t}^{\lambda\rho,\emph{ES}}\big]=
        &\mathbb{V}{\emph{ar}}\big[R_{t+1}+\gamma\bar{Q}_{t+1}-q^{\pi}(s,a)\big]
        \\
        \nonumber
        &+\gamma^{2}\lambda^{2}\mathbb{V}{\emph{ar}}\big[
        v^{\pi}(s^{'})-\bar{Q}_{t+1}\big]\\
        \nonumber
        &+\gamma^{2}\lambda^{2}\mathbb{V}{\emph{ar}}[\Delta_{t+1}]
        \\
        \label{variance-1}
        &
        +\gamma^{2}\lambda^{2}\mathbb{V}{\emph{ar}}\big[\rho_{t+1}\widetilde{G}_{t+1}^{\lambda\rho,\emph{ES}}\big],
    \end{flalign}
    where $\Delta_{t+1}=\bar{Q}_{t+1}-\rho_{t+1}Q_{t+1}-v^{\pi}(s^{'})+\rho_{t+1}q^{\pi}(s^{'},a^{'})$.
\end{theorem}
\begin{proof}
See Appendix F.
\end{proof}
Now, let's illustrate the significance of Eq.(\ref{variance-1}).

\textbf{(I)} It demonstrates total random sources lead to the variance. 
The first 3 terms reveal the variance of $\widetilde{G}_{t}^{\lambda\rho,\text{ES}}$ is cased by the following factors correspondingly: the error of one-step $\mathtt{Expected~Sarsa}$ for policy evaluation, the error between $\bar{Q}_{t+1}$ and true value $v^{\pi}$, and state-action transition randomness. The last term in (\ref{variance-1}) is the variance of future time.

\textbf{(II)} Please notice that if the CV term $\bar{Q}_{t+1}-\rho_{t+1}Q_{t+1}$ (in $\Delta_{t+1}$) vanishes, i.e. $\Delta_{t+1}=-v^{\pi}(s^{'})+\rho_{t+1}q^{\pi}(s^{'},a^{'})$, Eq.(\ref{variance-1}) is reduced to the recursive variance of $G_{t}^{\lambda\rho,\text{ES}}$ (\ref{off-es-recursive}).
Thus, by Eq.(\ref{variance-1}), comparing the variance of $\widetilde{G}_{t}^{\lambda\rho,\text{ES}}$ with $G_{t}^{\lambda\rho,\text{ES}}$ is equal to comparing the variance of $\Delta_{t+1}$.

Furthermore,
if a good estimator of $q^{\pi}$ is available, the two following events happen:
\begin{compactenum}
\item{For $\mathtt{ES}$($\lambda$)-$\mathtt{CV}$, the term $\Delta_{t+1}\approx 0$. Since for a proper estimate of $q^{\pi}$, the following happens
\[\bar{Q}_{t+1}-\rho_{t+1}Q_{t+1}\approx 0,-v^{\pi}(s^{'})+\rho_{t+1}q^{\pi}(s^{'},a^{'})\approx 0.\]
}
\item{While, for $\mathtt{ES}$($\lambda$), $\Delta_{t+1}=-v^{\pi}(s^{'})+\rho_{t+1}q^{\pi}(s^{'},a^{'})$, which is never be to $0$, no matter how good an estimate of $q^{\pi}$ we achieve.}
\end{compactenum}
Thus, if a good estimator of $q^{\pi}$ is available, we have,
\[
\underbrace{\mathbb{V}{\text{ar}}[\Delta_{t+1}]}_{\text{for $\mathtt{ES}(\lambda)$-$\mathtt{CV}$ iteration (\ref{es-recursive-cv})} }\ll\underbrace{\mathbb{V}{\text{ar}}[-v^{\pi}(s^{'})+\rho_{t+1}q^{\pi}(s^{'},a^{'})]}_{{\text{for $\mathtt{ES}(\lambda)$ iteration (\ref{off-es-recursive})} }}.
\] 
Thus $\widetilde{G}_{t}^{\lambda\rho,\text{ES}}$ enjoys a lower variance than ${G}_{t}^{\lambda\rho,\text{ES}}$.

\subsection{Numerical Analysis}
We use an experiment to verify that CV is efficient to reduce variance of $\mathtt{ES}$($\lambda$) for off-policy evaluation task.
In this experiment, the target policy $\pi$ is greedy policy, the value of $\pi$ is selected by $\mathtt{Q\text{-}learning}$ with $\epsilon_{k}$-greedy policy, where $\epsilon_{k}$ is decayed as $\epsilon_{k+1}=0.95\epsilon_{k}$, $\epsilon_{1}=0.2$.
After 150 episodes, $\epsilon_{150}\approx0$, and the value of target policy $\pi$ comes around $-20$.
We use $0.2$-greedy policy as behavior policy $\mu$. All algorithms use step-size $
        \alpha_{k}= 0.5$ and $\lambda=0.95$.
\begin{figure}[h]
    \centering   
    \includegraphics[scale=0.45]{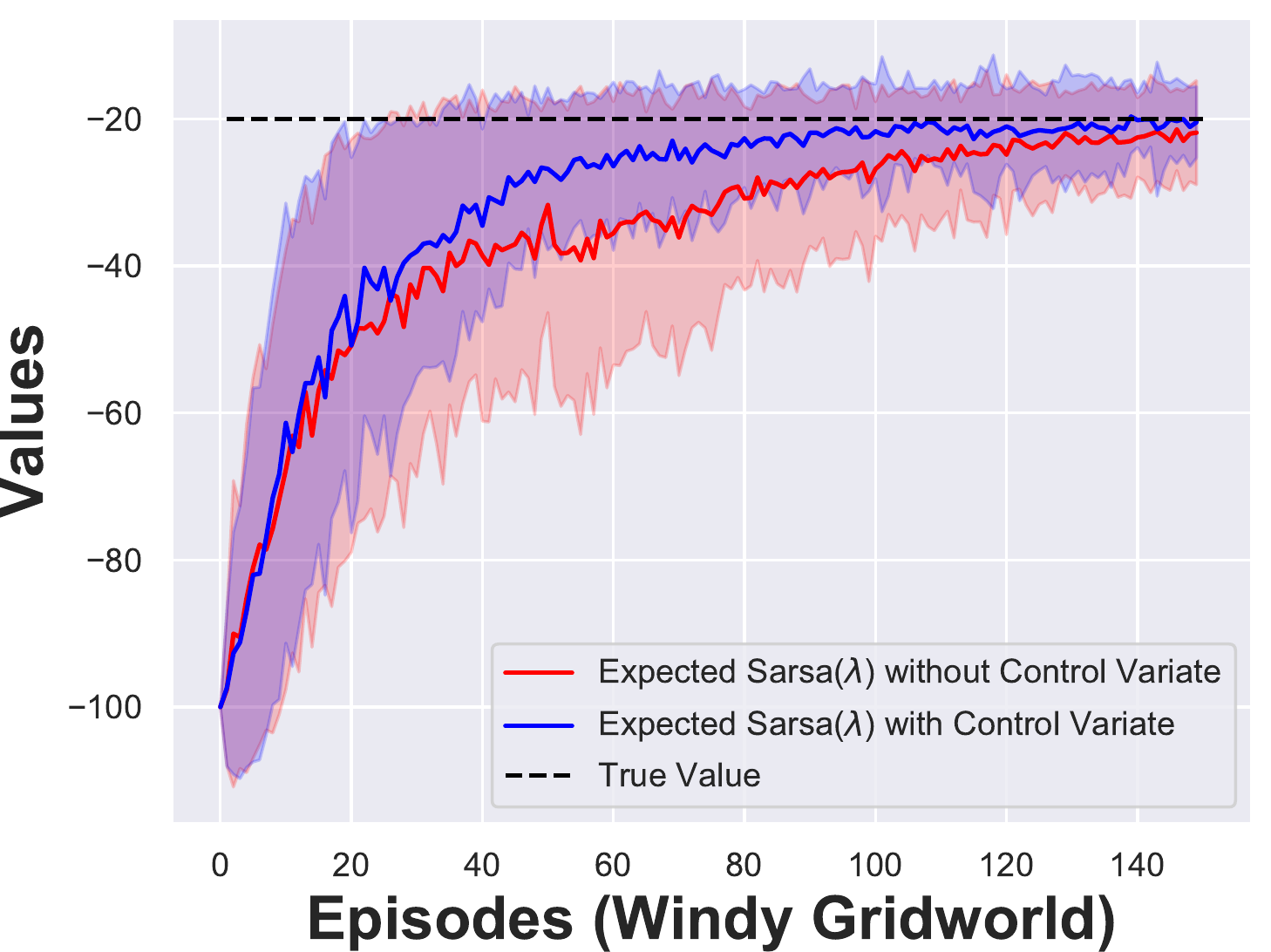}  
    \caption{
        Comparison the performance between $\mathtt{ES}$($\lambda$)-$\mathtt{CV}$ and $\mathtt{ES}$($\lambda$) for 
        off-policy evaluation task on windy gridworld.
        These unbroken lines are an average of 100 runs, and each run contains 150 episodes. 
         To preferably show variance during the learning process, we show the shadow width as the standard deviation.       
    }
\end{figure}

\section{Gradient Expected Sarsa($\lambda$)}
In this section, we extend $\mathtt{ES}$($\lambda$)-$\mathtt{CV}$ with linear function approximation.
Firstly, we prove the way to extend $\mathtt{ES}$($\lambda$)-$\mathtt{CV}$ with function approximation by \cite{sutton2018reinforcement} (section 12.9) is unstable.
Then, we propose a convergent gradient $\mathtt{Expected~Sarsa}$($\lambda$).

The Bellman equation (\ref{bellman-equation}) cannot be solved directly by tabular method for a large dimension of $\mathcal{S}$. 
We often use a parametric function to approximate
$
q^{\pi}(s,a)\approx\phi^{\top}(s,a)\theta={Q}_{\theta}(s,a),
$
where $\phi:\mathcal{S}\times\mathcal{A}\rightarrow\mathbb{R}^{p}$ is a \emph{feature map}.
Then ${Q}_{\theta}$ can be rewritten as a version of matrix
${Q}_{\theta}=\Phi\theta\approx q^{\pi},$
where $\Phi$ is a $|\mathcal{S}||\mathcal{A}|\times p$ matrix whose row is $\phi(s,a)$.
We assume that Markov chain induced by behavior policy $\mu$ is ergodic \cite{bertsekas2012dynamic}, 
i.e. there exists a stationary distribution $\xi$ such that
$\forall (S_{0},A_0)\in\mathcal{S}\times\mathcal{A}$, 
$\frac{1}{n}\sum_{k=1}^{n} P(S_{k}= s,A_{k}=a |S_{0},A_0)\overset{n\rightarrow\infty}\rightarrow \xi(s,a).$
We denote $\Xi$ as a $|\mathcal{S}|\times|\mathcal{A}|$ diagonal matrix whose diagonal element is 
$\xi(s,a)$.  
\begin{figure}[t]
    \centering
    \includegraphics[scale=0.5, trim={1mm 1mm 1mm 1mm}, clip]{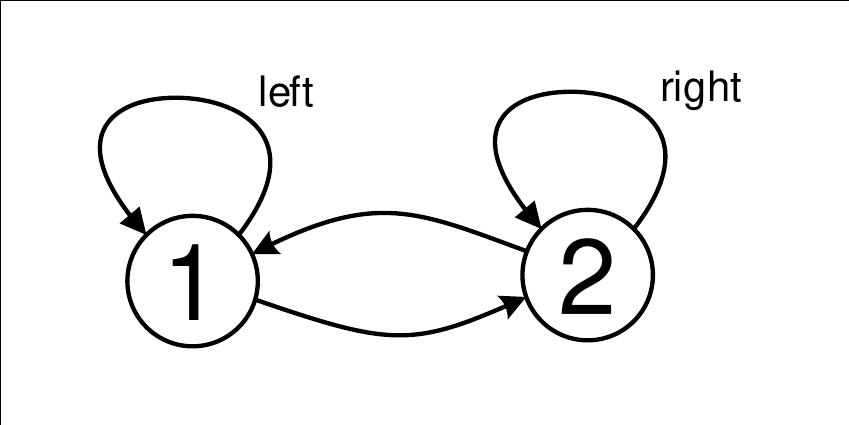}
    \caption{Two-state Example. We assign the features $ \{(1, 0)^{\top}, (2, 0)^{\top}, (0, 1)^{\top}, (0, 2)^{\top}\}$ to the state-action pairs $\{(1,\mathtt{right}),(2,\mathtt{right}),(1,\mathtt{left}),(2,\mathtt{left})\}$, $\pi(\mathtt{right} |\cdot)=1$ and $\mu(\mathtt{right} | \cdot)=0.5$.
    }
\end{figure}

\subsection{Instability of ES($\lambda$) with Function Approximation}
A typical update to extend (\ref{ES-Gtt+n}) has been presented in~\cite{sutton2018reinforcement} (section 12.9),
\begin{flalign}
\nonumber
    \theta_{t+1}&=\theta_{t}+\alpha_{t}(
    \widetilde{G}_{t,\theta}^{\lambda\rho,\text{ES}}-Q_{\theta}(S_t,A_t)
    )\nabla Q_{\theta}(S_{t},A_t)\\
        \label{gradient-ES}
    &=\theta_{t}+\alpha_{t}(\sum_{l=t}^{\infty}(\gamma\lambda)^{l-t}\delta_{l,\theta}^{\text{ES}}\rho_{t+1:l})\phi_t,
\end{flalign}
where $\alpha_{t}$ is step-size, $\delta_{l,\theta}^{\text{ES}}=R_{l+1}+\gamma\theta^{\top}_{l}\mathbb{E}_{\pi}[\phi(S_{l+1},\cdot)]-\theta^{\top}_{l}\phi_l$, $\phi_l$ is short for $\phi(S_{l},A_{l})$. Once the system (\ref{gradient-ES}) has reached a stable state, for any $\theta_t$, the expected parameter can been written as
\begin{flalign}
\label{expected-para-es}
\mathbb{E}[\theta_{t+1}|\theta_t]=\theta_{t} + \alpha_{t}(A\theta_t+b),
\end{flalign}
where
\begin{flalign}
A&= \Phi^{\top}\Xi(I-\gamma\lambda P^{\mu})^{-1}(\gamma P^{\pi}-I)\Phi,\\
b&= \Phi^{\top}\Xi(I-\gamma\lambda P^{\mu})^{-1}r, r=\mathbb{E}[R_{t+1}|S_{t},A_{t}].
\end{flalign}
If the system (\ref{expected-para-es}) converges, then $\theta_{t}$ converges to the \emph{TD fixed point} $\theta^{*}$ that satisfies
$
    A\theta^{*}+b=0.
$

\textbf{What condition guarantees the convergence of the (\ref{gradient-ES})/ (\ref{expected-para-es})?}
Unfortunately, the instability of (\ref{gradient-ES}) for off-policy is firstly realized by Sutton and Barto\shortcite{sutton2018reinforcement}, but it is only an intuitive guess inspired by previous works. Now, we provide a simple but rigorous theoretical analysis to illustrate the divergence of Eq.(\ref{gradient-ES}).
It is known that for on-policy learning $\mu=\pi$, $A$ is a negative definite matrix \cite{tsitsiklis1997analysis}.
Thus, for on-policy learning, (\ref{gradient-ES}) converges to $-A^{-1}b$.
However, for off-policy learning, since the steady state-action distribution does not match the transition probability and $P^{\pi}\xi\ne\xi$, which results in, there is no guarantee that $A$ is a negative definite matrix~\cite{tsitsiklis1997analysis}. 
Thus (\ref{gradient-ES}) may diverge.

\textbf{An Unstable Example} Now, we use a typical example \cite{touati2018convergent} to illustrate the instability of iteration (\ref{gradient-ES}).
The state transition of the example is presented in Figure 2.
After some simple algebra (the detail is provided in Appendix G), we have {\small{$A=\begin{pmatrix}
\frac{6\gamma-\gamma\lambda-5}{2(1-\gamma\lambda)} & 0 \\
\frac{3\gamma}{2} &- \frac{5}{2}
\end{pmatrix}$}}.
For any $\theta_{0}=(\theta_{0,1},\theta_{0,2})^{\top}$, a positive constant step-size $\alpha$, 
according to (\ref{expected-para-es}), 
we have 
\begin{flalign}
\label{example-iteration}
\mathbb{E}[\theta_{t+1}|\theta_t]\overset{\text{def}}=&(\theta_{t+1,1},\theta_{t+1,2})^{\top},\\
\label{example-iteration-1}
\theta_{t+1,1}=&\theta_{0,1}\prod_{l=0}^{t}(1+\alpha \frac{6\gamma-\gamma\lambda-5}{2(1-\gamma\lambda)}),\\
\label{example-iteration-2}
\theta_{t+1,2}=&\theta_{0,2}\prod_{l=0}^{t}(1-\alpha\frac{5}{2})^{\top}
\end{flalign}
For any $\lambda\in(0,1)$, $\gamma\in(\frac{5}{6-\lambda},1)$, $\frac{6\gamma-\gamma\lambda-5}{2(1-\gamma\lambda)}$ is a positive scalar. Since then $A$ cannot be a negative matrix.
Furthermore, according to (\ref{example-iteration-1}),
\[|\theta_{t+1,1}|=|\theta_{0,1}||(1+\alpha \frac{6\gamma-\gamma\lambda-5}{2(1-\gamma\lambda)})^{t+1}|\rightarrow+\infty.\]

\subsection{Convergent Algorithm}

The above discussion of the instability for off-policy learning shows that we should abandon the way presented in (\ref{gradient-ES}).
In this section, we propose a convergent gradient $\mathtt{ES}$($\lambda$) algorithm.

We solve the problem by mean square projected Bellman equation (MSPBE) \cite{sutton2009fast_a},
\[
\text{MSPBE}(\theta,\lambda)=\frac{1}{2}\|\Phi\theta-\Pi\mathcal{B}^{\pi}_{\lambda}(\Phi\theta)\|^{2}_{\Xi},   
\]
where $\Pi = \Phi(\Phi^{T}\Xi\Phi)^{-1}\Phi^{T}\Xi$ is an 
$|\mathcal{S}| \times |\mathcal{S}|$ projection matrix.
Furthermore, MSPBE$(\theta,\lambda)$ can be rewritten as,
\begin{flalign}
\label{Eq:mspbe}
\min_{\theta}\text{MSPBE}(\theta,\lambda)=\min_{\theta}\frac{1}{2}\|A\theta+b\|^{2}_{M^{-1}},
\end{flalign}
where $M=\mathbb{E}[\phi_{t}\phi_{t}^{\top}]=\Phi^{T}\Xi\Phi$.
The derivation of (\ref{Eq:mspbe}) is provided in Appendix H.

The computational complexity of the invertible matrix $M^{-1}$ is at least $\mathcal{O}(p^3)$~\cite{golub2012matrix}, where $p$ is the dimension of feature space. Thus, it is too expensive to use gradient updates to solve the problem (\ref{Eq:mspbe}) directly.
Besides, as pointed out in~\cite{szepesvari2010algorithms,liu2015finite}, we cannot get an unbiased estimate of $\nabla_{\theta}\text{MSPBE}(\theta,\lambda)=A^{\top}M^{-1}(A\theta+b)$. 
In fact, since the update law of gradient involves the product of expectations, 
the unbiased estimate cannot be obtained via a single sample. It needs to sample twice, which is a double sampling problem. Secondly, $M^{-1}=\mathbb{E}[\phi_{t} \phi_{t}^T]^{-1}$ cannot also be estimated via a single sample, which is the second bottleneck of applying stochastic gradient method to solve problem (\ref{Eq:mspbe}).

A practical way is converting (\ref{Eq:mspbe}) to be a convex-concave saddle-point problem~\cite{liu2015finite}.
For $f: \mathbb{R}^{d}\rightarrow\mathbb{R}$, its \emph{convex conjugate} \cite{bertsekas2009convex} function 
$f^{*} : \mathbb{R}^{d}\rightarrow\mathbb{R}$ is defined as
\[f^{*}(y)=\sup_{x\in\mathbb{R}^{d}}\{y^{T}x-f(x)\}.\]
By $(\frac{1}{2}\|x\|^{2}_{M})^{*}= \frac{1}{2}\|y\|^{2}_{M^{-1}}$, we have
$
\frac{1}{2}\|y\|^{2}_{M^{-1}}=\max_{\omega}(y^{T}\omega-\frac{1}{2}\|\omega\|^{2}_{M}).
$
Thus, (\ref{Eq:mspbe})
is equivalent to the next convex-concave saddle-point problem
\begin{flalign}
\label{saddle-point-problem}
\min_{\theta}\max_{\omega}\{(A\theta+b)^{\top}\omega-\frac{1}{2}\|\omega\|^{2}_{M}\}.
\end{flalign}
It is easy to see that
if $(\theta^{*},\omega^{*})$ is the solution of problem (\ref{saddle-point-problem}), then $\theta^{*}=\arg\min_{\theta}\text{MSPBE}(\theta,\lambda)$.
In fact, let $\omega^{*}=\arg\max_{\omega}{(A\theta+b)^{\top}\omega-\frac{1}{2}\|\omega\|_{M}^{2}}$, then $\omega^{*}=M^{-1}(A\theta+b)$.
Taking $\omega^{*}$ into (\ref{saddle-point-problem}), then (\ref{saddle-point-problem}) is reduced to $\min_{\theta}\frac{1}{2}\|A\theta+b\|^{2}_{M^{-1}}$, which illustrates that the solution of (\ref{Eq:mspbe}) contained in (\ref{saddle-point-problem}).
Gradient update is a natural way to solve problem (\ref{saddle-point-problem}) (ascending in $\omega$ and descending in $\theta$) as follows, 
\begin{flalign}
\label{E-al-0}
    \omega_{t+1}&=\omega_t+\beta_t(A\theta_t+b-M\omega_t),\\
    \label{E-al}
    \theta_{t+1}&=\theta_t-\alpha_t A^{\top}\omega_t,
\end{flalign}
where $\alpha_t,\beta_t$ is step-size, $t\ge0$.

\textbf{Stochastic On-line Implementation}
However, since $A, b$, and $ M$ are versions of expectations, for model-free RL, we can not get the probability of transition. A practical way is to find the unbiased estimators of them.
Let $e_{0}=0, \rho_{t}=\frac{\pi(A_{t}|S_{t})}{\mu(A_{t}|S_{t})}, e_{t}=\lambda\gamma \rho_{t}e_{t-1}+\phi_{t}, \hat{b}_{t}=R_{t+1}e_{t}, \hat{A}_{t}=e_{t}(\gamma\mathbb{E}_{\pi}[\phi(S_{t+1,\cdot})]-\phi_{t})^{\top},\hat{M}_{t}=\phi_{t}\phi^{\top}_{t}$. By Theorem 9 in \cite{maei2011gradient}, we have  
\[\mathbb{E}[\hat{A}_{t}]=A,\mathbb{E}[\hat{b}_{t}]=b,\mathbb{E}[\hat{M}_{t}]=M.\]
Replacing the expectations in (\ref{E-al-0}) and (\ref{E-al}) by corresponding unbiased estimates, we define the stochastic on-line implementation of (\ref{E-al-0}) and (\ref{E-al}) as follows,
\begin{flalign}
\label{stochastic-im-1}
\omega_{t+1}&=\omega_{t}+\beta_t(\hat A_{t}{\theta}_{t}+\hat b_{t}-\hat M_{t}\omega_{t}) ,\\
\label{stochastic-im}
\theta_{t+1}&=\theta_{t}-\alpha_t \hat A_{t}^{\top}\omega_{t}.
\end{flalign}
More details are summarized in Algorithm \ref{alg:algorithm1}.

\begin{algorithm}[tb]
    \caption{Gradient Expected Sarsa$(\lambda)$}
    \label{alg:algorithm1}
    \begin{algorithmic}
        \STATE { \textbf{Initialization}: $\omega_{0}=0$, ${\theta}_{0}=0$, $\alpha_0>0,\beta_0>0$} \\
        \FOR{$i=0$ {\bfseries to} $n$} 
        \STATE ${e}_{-1}={0}$
        \FOR{$t=0$ {\bfseries to} $T_{i}$}
        \STATE Observe $\{S_{t},A_{t},R_{t+1},S_{t+1},A_{t+1}\}\sim\mu$
        \STATE {$e_{t}=\lambda\gamma \rho_{t}e_{t-1}+\phi_{t}, \text{where}~\rho_{t}=\frac{\pi(A_{t}|S_{t})}{\mu(A_{t}|S_{t})}$}
        \STATE $\delta_{t}=R_{t+1}+\gamma{\theta}_{t}^{\top}\mathbb{E}_{\pi}\phi(S_{t+1},\cdot)-{\theta}_{t}^{\top}\phi_t$
        \STATE $\omega_{t+1}=\omega_{t}+\beta_t(e_{t}\delta_t-\phi_t\phi_{t}^{\top}\omega_t)$
        \STATE $\theta_{t+1}=\theta_{t}-\alpha_t (\gamma\mathbb{E}_{\pi}[\phi(S_{t+1,\cdot})]-\phi_{t})e_t^{\top}\omega_{t}$
        \ENDFOR
        \ENDFOR
        \STATE { \textbf{Output}:${\theta}$}
    \end{algorithmic}
\end{algorithm}

\subsection{Convergence Analysis}

\begin{table*}
\centering
 \label{table}  
    \begin{tabular}{lllll}

        \toprule
        Algorithm     & Reference &Step-size & Convergence Rate \\
        \midrule
        $\mathtt{TD}(0)$&\cite{Nathaniel2015ontd}&$\alpha_{t}=\mathcal{O}(\frac{1}{t^{\eta}})$, $\eta\in(0,1)$&$\mathcal{O}(1/\sqrt{T})$\\
        $\mathtt{TD}(0)$&\cite{gal2018finite} &$\sum_{t=1}^{\infty}\alpha_t=\infty$&$\mathcal{O}(e^{-\frac{\sigma}{2}T^{1-\eta}}+\frac{1}{T^{\eta}})$\\
        $\mathtt{GTD}(0)$&\cite{gal2018finitesample}&
        $\sum_{t=1}^{\infty}\alpha_t=\infty$,$\frac{\beta_{t}}{\alpha_{t}}\rightarrow0$
        &$\mathcal{O}(({1}/{T})^{\frac{1-\kappa}{3}})$\\
        $\mathtt{GTD}$ & \cite{liu2015finite}&constant step-size & $\mathcal{O}(1/\sqrt{T})$     \\
        $\mathtt{GTD}$     & \cite{wang2017finite} &$\sum_{t=1}^{\infty}\alpha_t=\infty$, $\frac{\sum_{t=1}^{T}\alpha^{2}_{t}}{\sum_{t=1}^{T}\alpha_{t}}\leq\infty$ & $\mathcal{O}(1/\sqrt{T})$     \\
        $\mathtt{GTB/GRetrace}$    & \cite{touati2018convergent}   & $\alpha_{t},\beta_{t}=\mathcal{O}(\frac{1}{t})$   & $\mathcal{O}(1/{T})$  \\
        {\color{red}{Ours} }    &    &    {\color{red}{constant step-size }} & ${\color{red}{\mathcal{O}(1/{T}) }}$   \\
        \bottomrule
    \end{tabular}
     \caption{
         Convergence Rate of Gradient Temporal Difference Learning
         }     

        \end{table*}

We measure the convergence rate of problem (\ref{saddle-point-problem}) by \emph{primal-dual gap error}
\cite{nemirovski2009robust}.
Let \[\Psi(\theta,\omega)=(A\theta+b)^{T}\omega-\frac{1}{2}\|\omega\|^{2}_{M},\]
    the primal-dual gap error at each solution $(\omega,\theta)$ is 
    \[
    \epsilon_{\Psi}(\theta,\omega)=\max_{\omega^{'}} \Psi(\theta,\omega^{'})- \min_{\theta^{'}} \Psi(\theta^{'},\omega).
    \]
\begin{theorem}[Convergence of Algorithm \ref{alg:algorithm1}]
    \label{theo:on-algo2-convergence}
    Consider the sequence $\{(\theta_{t},\omega_{t})\}_{t=1}^{T}$ generated by (\ref{stochastic-im}), step-size $\alpha,\beta$ are positive constants.
    Let $(\theta^*,\omega^*)$ be the optimal solution of (\ref{saddle-point-problem}), $\bar{\theta}_{T}=\frac{1}{T}(\sum_{t=1}^{T}\theta_{t})$,
    $\bar{\omega}_{T}=\frac{1}{T}(\sum_{t=1}^{T}\omega_{t})$ and we choose the step-size $\alpha,\beta$ satisfy $1-\sqrt{\alpha\beta}\|A\|_{*}>0$, where $\|A\|_{*}=\sup_{\|x\|=1}\|Ax\|$ is operator norm.
    If parameter $(\theta,\omega)$ is on a bounded
    $D_{\theta} \times D_{\omega}$, 
    i.e \emph{diam} $D_{\theta}=\sup\{\|\theta_{1}-\theta_{2}\|;\theta_{1},\theta_{2}\in D_{\theta}\}\leq\infty$, \emph{diam} $D_{\omega}$$\leq\infty$, 
    $\mathbb{E}[\epsilon_{\Psi}(\bar{\theta}_{T},\bar{\omega}_{T})]$ is upper bounded by:
    \begin{flalign}
    \nonumber
    \sup_{(\theta,\omega)}\{\dfrac{1}{T}(\dfrac{\|\theta-\theta_{0}\|^{2}}{2\alpha}+\dfrac{\|\omega-\omega_{0}\|^{2}}{2\beta})\}.
    \end{flalign}
\end{theorem}
\begin{proof} 
See Appendix I.
\end{proof}
\begin{remark}
\label{remark-4}
Theorem \ref{theo:on-algo2-convergence} illustrates \textbf{(I)}
when $\alpha=\beta=\mathcal{O}(\frac{1}{\sqrt{T}})$, then the overall convergence rate of $\mathbb{E}[\epsilon_{\Psi}(\bar{\theta}_{T},\bar{\omega}_{T})]$ is $\mathcal{O}(\frac{1}{\sqrt{T}})$, which reaches the worst rate of black box oriented sub-gradient methods \cite{nesterovintroductory2004}; \textbf{(II)} when $\alpha=\beta =\mathcal{O}(1)$, a positive scalar, then $\mathbb{E}[\epsilon_{\Psi}(\bar{\theta}_{T},\bar{\omega}_{T})]=\mathcal{O}(\dfrac{1}{T}).$
\end{remark}

\subsection{Related Works and Comparison}
Liu et al.\shortcite{liu2015finite} firstly derives $\mathtt{GTD}$ via convex-concave saddle-point formulation, and they prove the convergence rate reaches $\mathbb{E}[\epsilon_{\Psi}(\tilde{\theta}_{T},\tilde{\omega}_{T})]=\mathcal{O}(\frac{1}{\sqrt{T}})$, where $\tilde{\theta}_{T}$ is Polyak-average: $\tilde{\theta}_{T}=\dfrac{\sum_{t=1}^{T}\alpha_{t}\theta_{t}}{\sum_{t=1}^{T}\alpha_{t}}$, $\tilde{\omega}_{T}=\dfrac{\sum_{t=1}^{T}\alpha_{t}\omega_{t}}{\sum_{t=1}^{T}\alpha_{t}}$. 
Their $\mathtt{GTD}$ requires each $\theta_{t},\omega_t$ is projected into the space $D_{\theta},D_{\omega}$. 
Later, Wang et al.\shortcite{wang2017finite} extends the work of Liu et al.\shortcite{liu2015finite}, they suppose the data is generated from Markov processes rather than I.I.D assumption. Wang et al.\shortcite{wang2017finite} prove the convergence rate $\mathbb{E}[\epsilon_{\Psi}(\tilde{\theta}_{T},\tilde{\omega}_{T})]=\mathcal{O}(\frac{\sum_{t=1}^{T}\alpha^{2}_{t}}{\sum_{t=1}^{T}\alpha_{t}})$, the best convergence rate reaches $\mathcal{O}(\frac{1}{\sqrt{T}})$, where the step-size satisfies $\sum_{t=1}^{\infty}\alpha_t=\infty$, $\frac{\sum_{t=1}^{T}\alpha^{2}_{t}}{\sum_{t=1}^{T}\alpha_{t}}\leq\infty$ and $(\tilde{\theta}_{T},\tilde{\omega}_{T})$ 
is also Polyak-average, the same as \cite{liu2015finite}.
Besides, the $\mathtt{GTD}$ of Wang et al.\shortcite{wang2017finite} also require projecting the parameter into the space $D_{\theta},D_{\omega}$.

Both Polyak-averaging and projection make the implementation of gradient TD learning more difficult. 
Comparing with \cite{liu2015finite,wang2017finite} , our $\mathtt{GES(\lambda)}$ removes Polyak-averaging and projection, while reaches a faster convergence rate.

Recently, \cite{gal2018finitesample} proves $\mathtt{GTD(0)}$ family \cite{sutton2009fast_a,sutton2009convergent_b} converges at $\mathcal{O}((\frac{1}{T})^{\frac{1-\kappa}{3}})$, but nerve reach $\mathcal{O}(\frac{1}{T})$, where $\kappa\in(0,1)$.
Nathaniel and Prashanth  \shortcite{Nathaniel2015ontd} proves $\mathtt{TD(0)}$ \cite{sutton1988learning} converges at $\mathcal{O}(\frac{1}{\sqrt{T}})$ with step-size $\alpha_{t}=\mathcal{O}(\frac{1}{t^{\eta}})$, where $\eta\in(0,1)$.
Then, Dalal et al.\shortcite{gal2018finite} further explores the property of $\mathtt{TD(0)}$, and they prove the convergence rate achieves
$\mathcal{O}(e^{-\frac{\sigma}{2}T^{1-\eta}}+\frac{1}{T^{\eta}})$, but never reach $\mathcal{O}(\frac{1}{T})$, where $\eta\in(0,1)$, $\sigma$ is the minimum eigenvalue of the matrix $A^{\top}+A$.

 Comparing to the all above works, we improve the optimal convergence rate to $\mathcal{O}(\dfrac{1}{T})$ with 
 a more relaxed step-size than theirs.
 Besides, although the $\mathtt{GTB}(\lambda)$/$\mathtt{GRetrace}(\lambda)$ \cite{touati2018convergent} reaches the same convergence rate as ours, their result depends on a decay step-size.

 More details of the convergence rate of gradient temporal difference learning are summarized in Table 1.

\section{Experiments}

\begin{figure*}[t!]
    \centering
    \subfigure
    {\includegraphics[width=5.5cm,height=4cm]{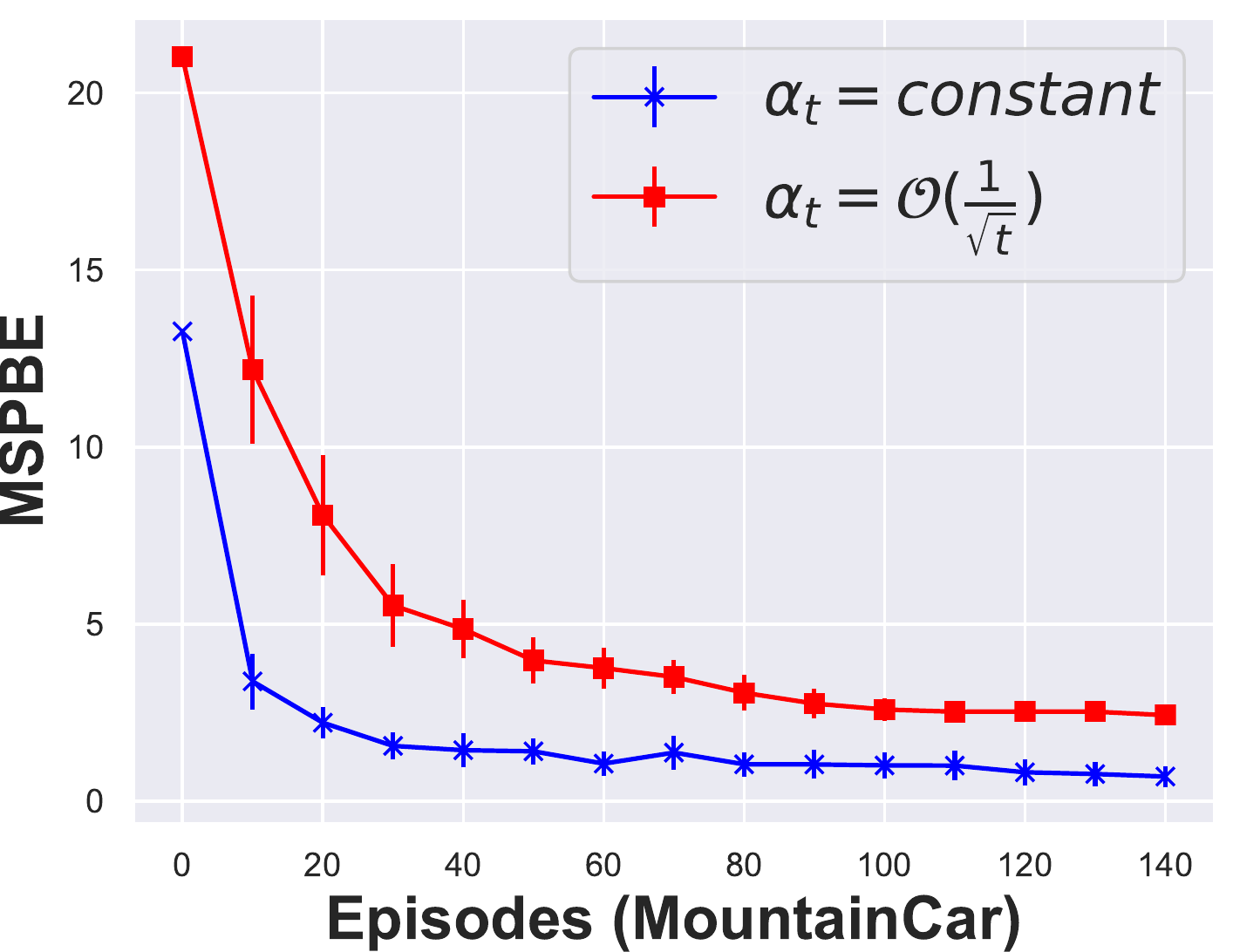}}
    \subfigure
    {\includegraphics[width=5.5cm,height=4cm]{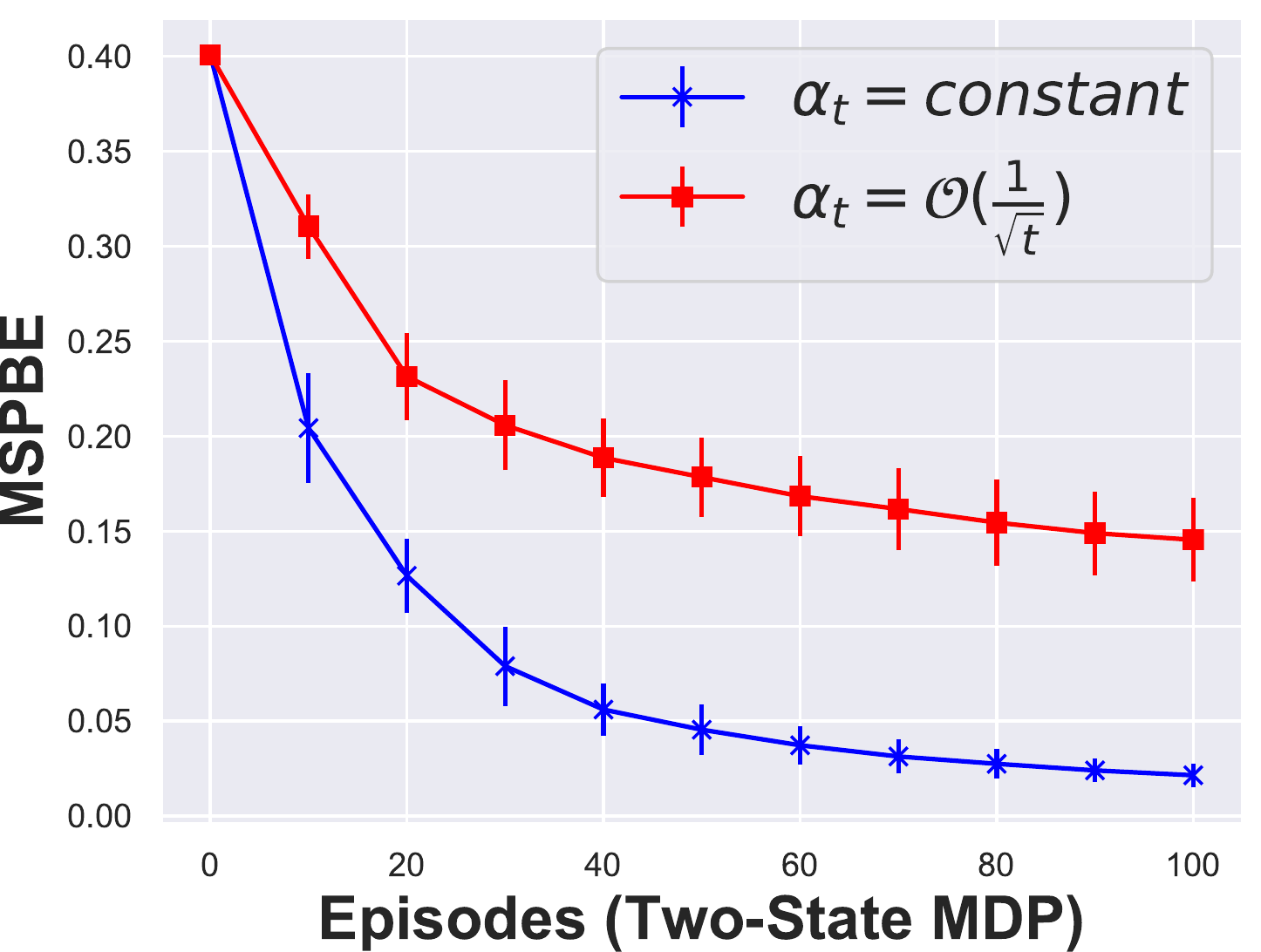}}
    \subfigure
    {\includegraphics[width=5.5cm,height=4cm]{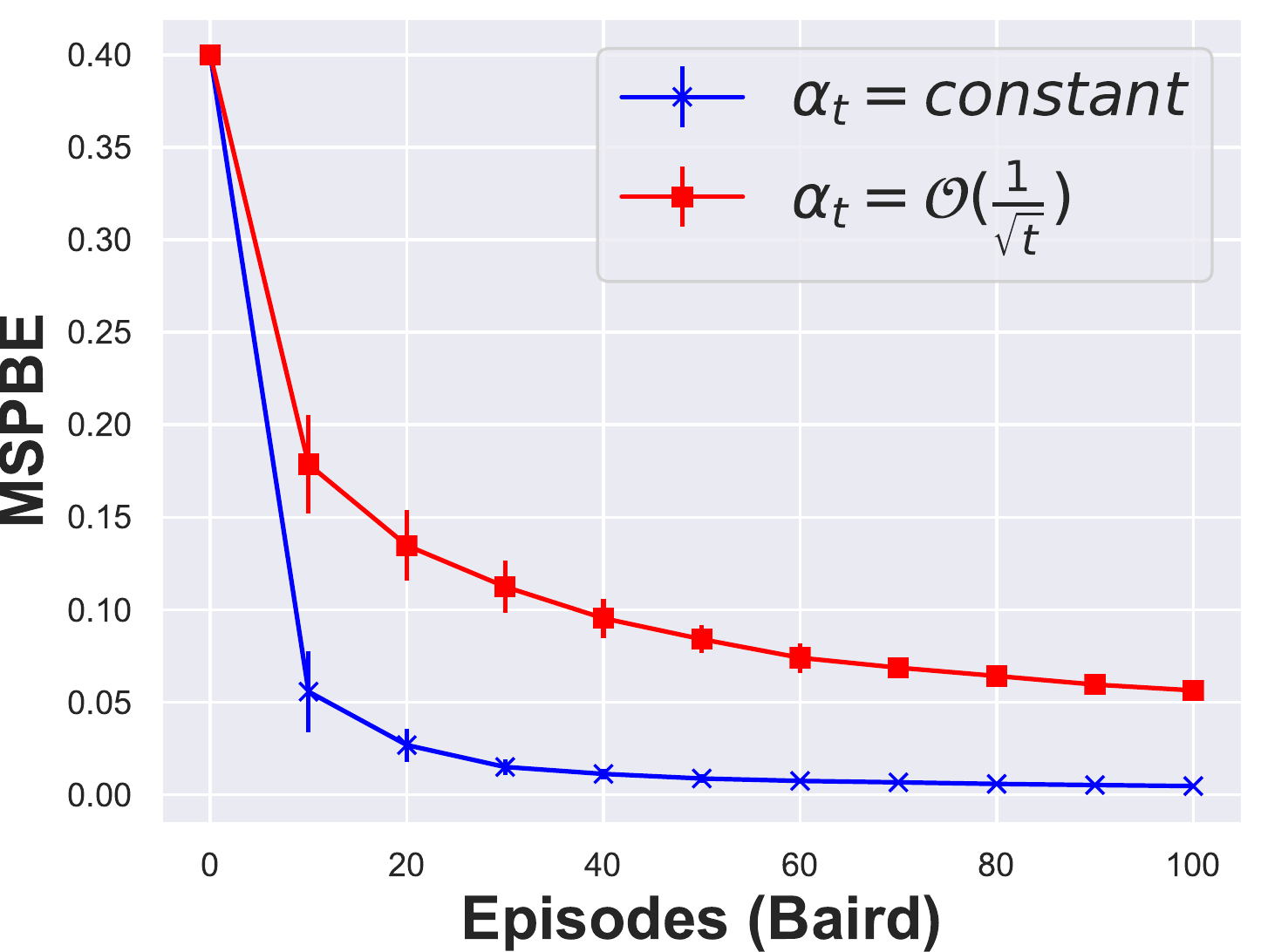}}
    \caption
    {
         Comparison between the  constant step-size and the decay step-size $\frac{1}{\sqrt{t}}$ for $\mathtt{GES}(\lambda)$.
    }
\end{figure*}
\begin{figure*}[t!]
    \centering
    \subfigure
    {\includegraphics[width=5.5cm,height=4cm]{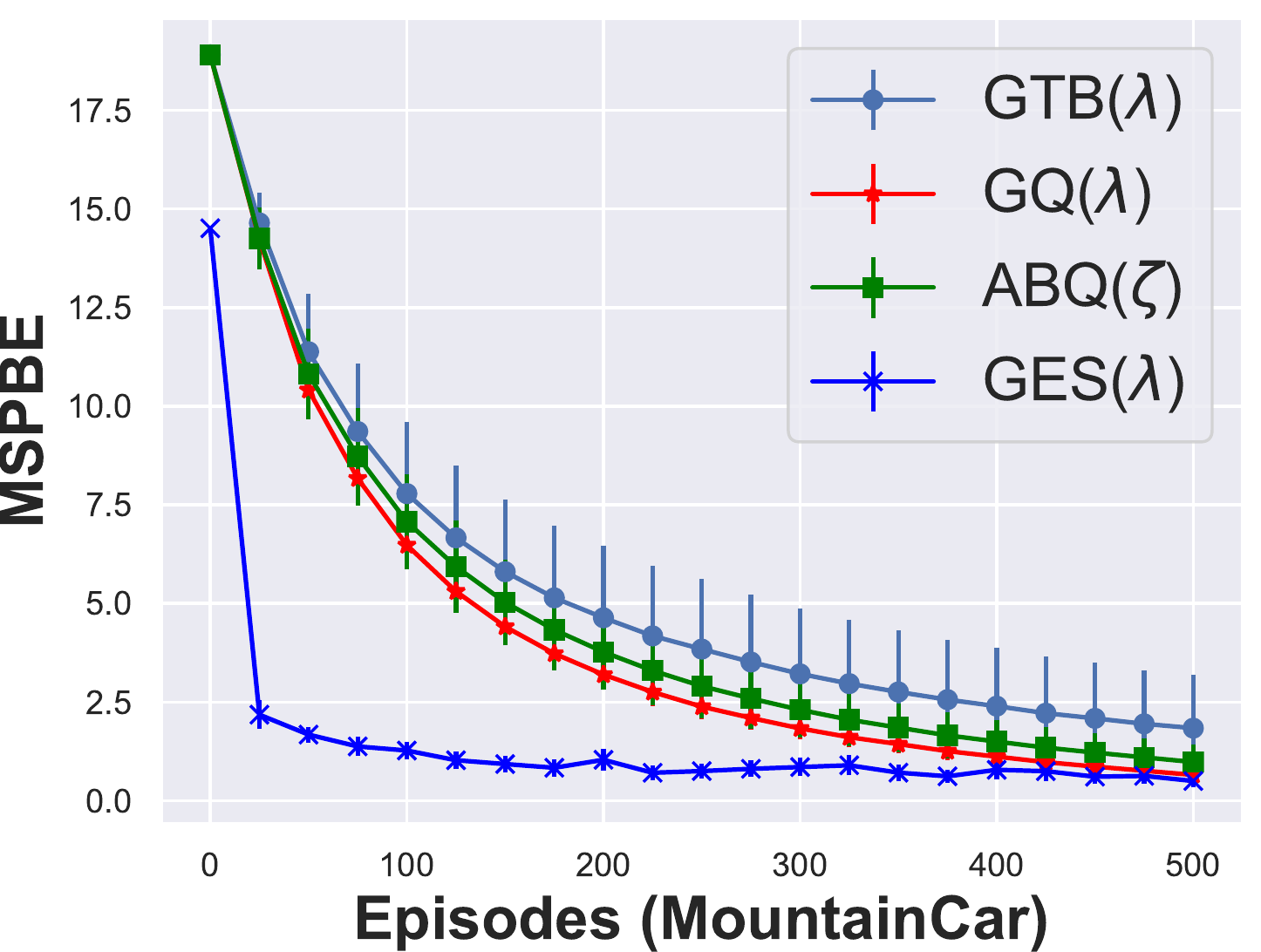}}
    \subfigure
    {\includegraphics[width=5.5cm,height=4cm]{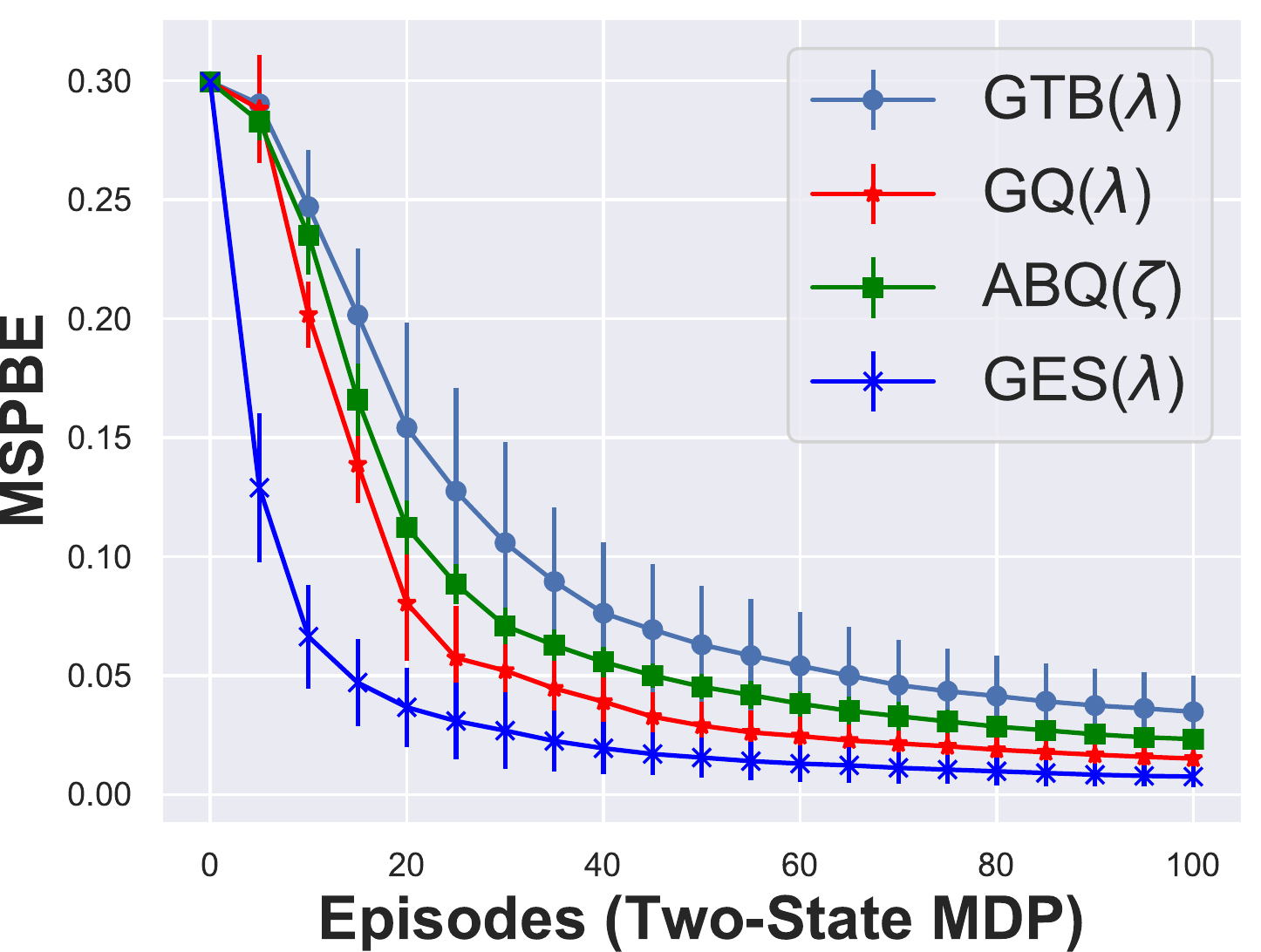}}
    \subfigure
    {\includegraphics[width=5.5cm,height=4cm]{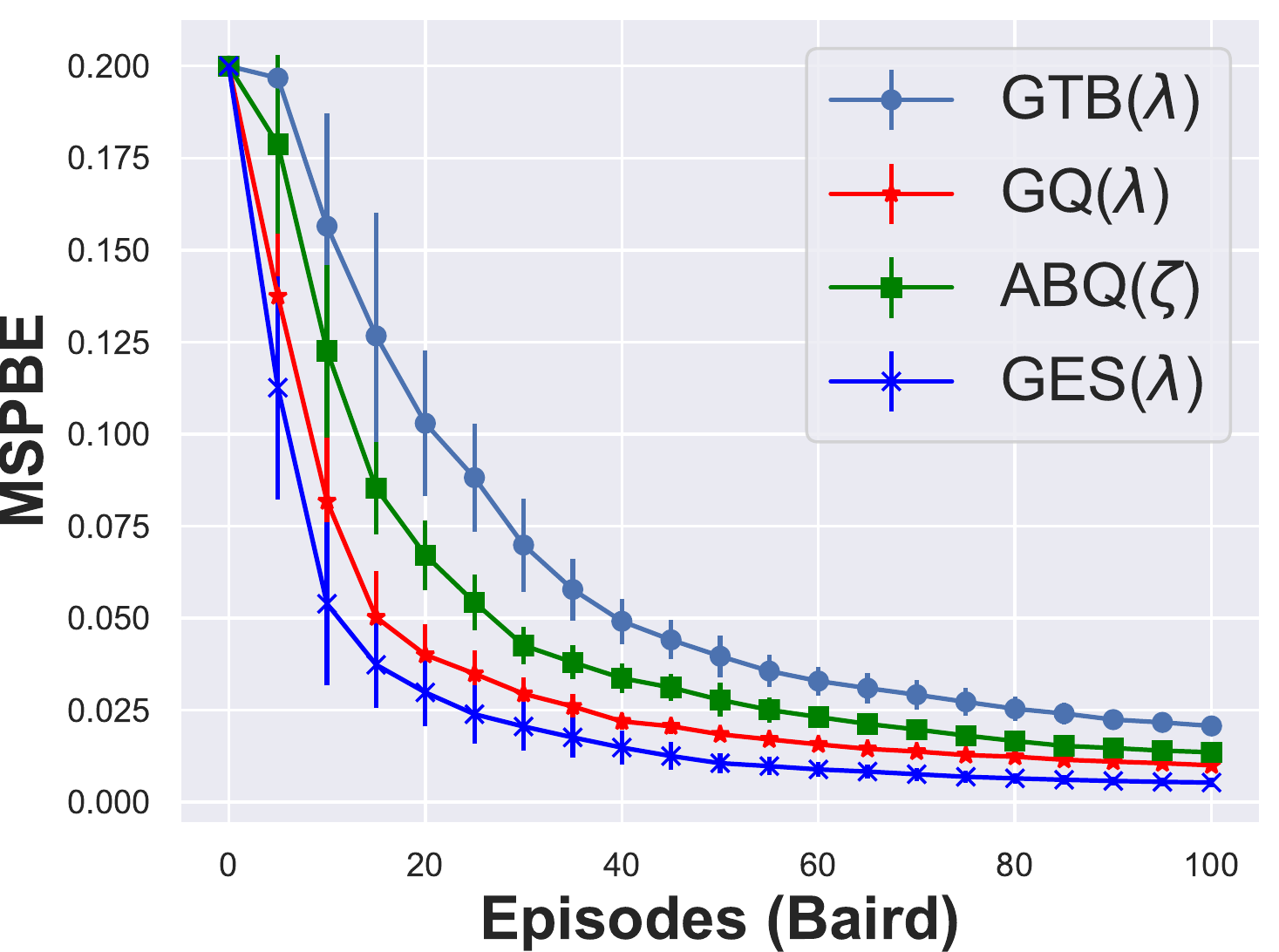}}
    \caption
    {
        MSPBE comparison over episode.
    }
\end{figure*}
\begin{figure*}[t!]
    \centering
    \subfigure
    {\includegraphics[width=5.5cm,height=4cm]{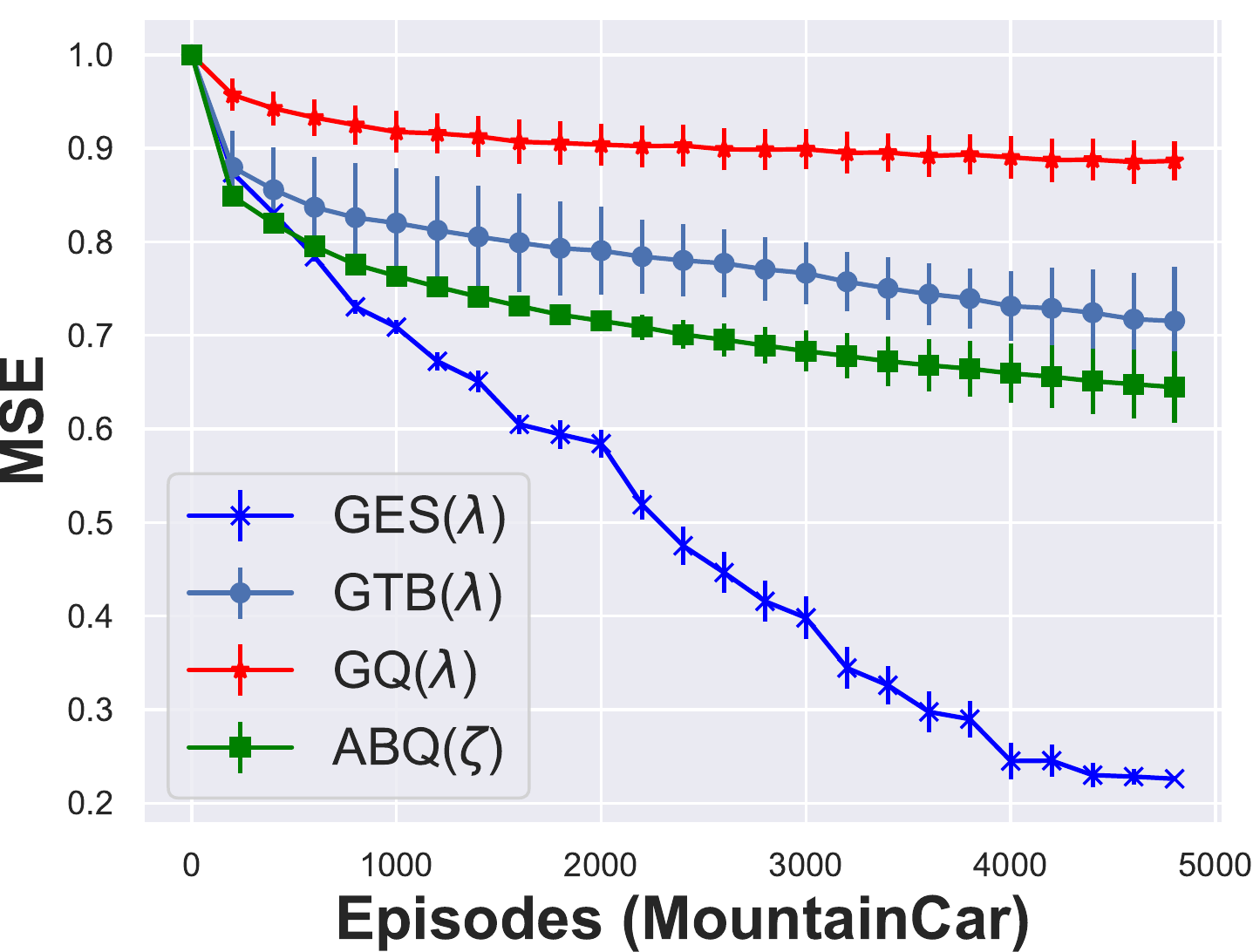}}
    \subfigure
    {\includegraphics[width=5.5cm,height=4cm]{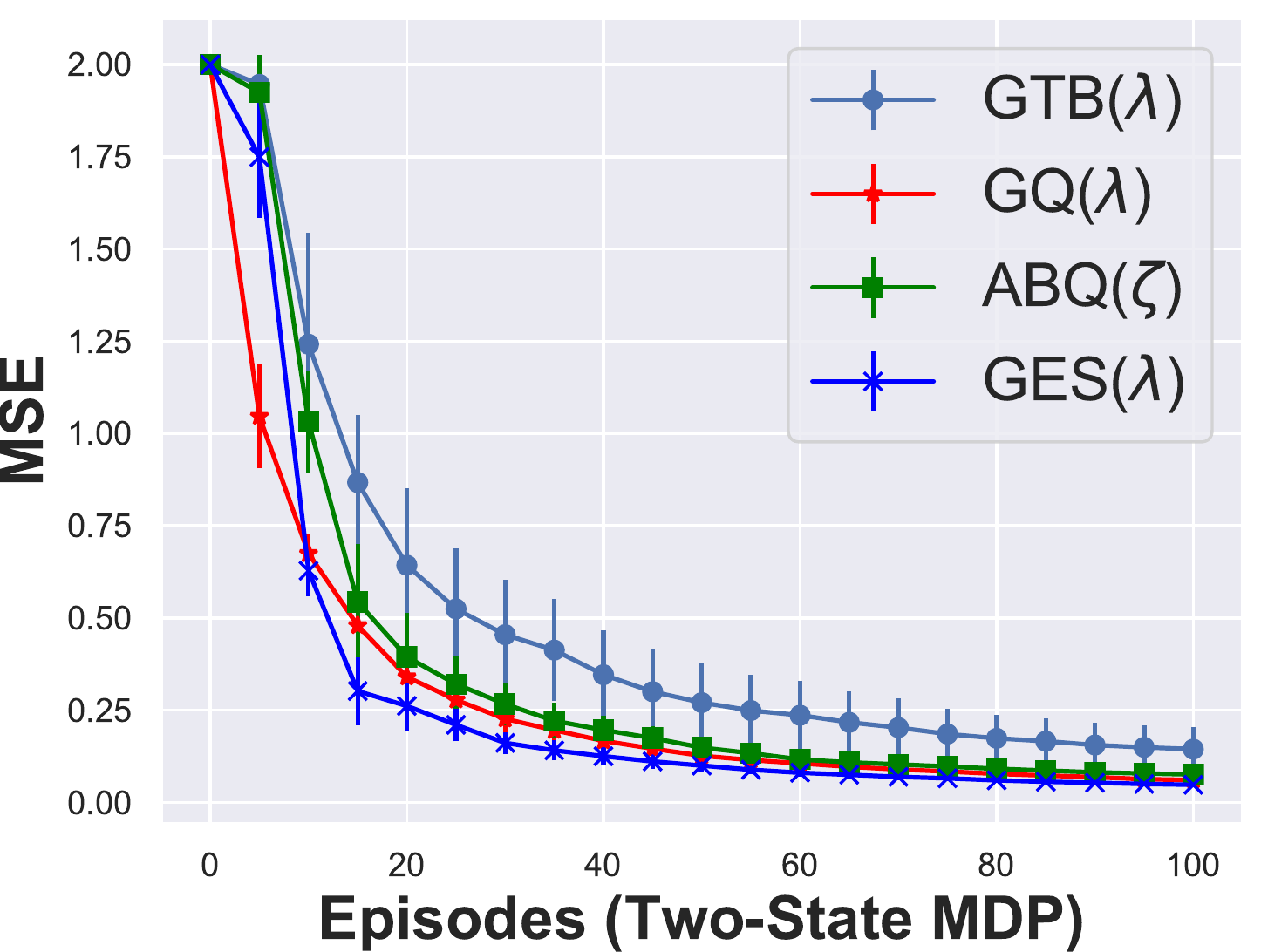}}
    \subfigure
    {\includegraphics[width=5.5cm,height=4cm]{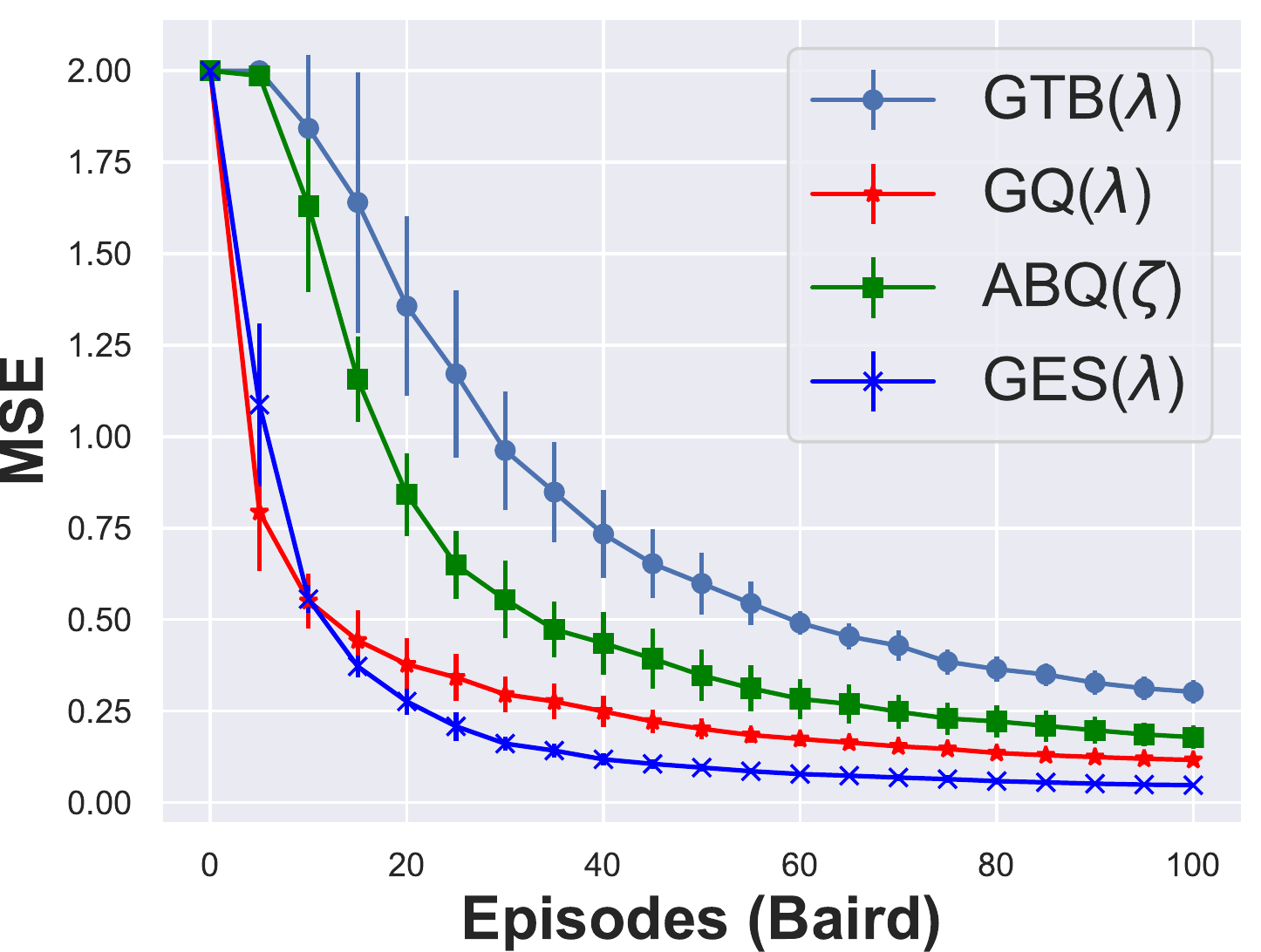}}
    \caption
    {
        MSE comparison over episode.
    }
\end{figure*}

In this section, we employ three typical domains to test the capacity of $\mathtt{GES}(\lambda)$ for off-policy evaluation,
\emph{Mountaincar}, \emph{Baird Star} \cite{baird1995residual}, and {Two-state MDP} \cite{touati2018convergent}. We compare $\mathtt{GES}(\lambda)$ with the three state-of-art algorithms:
$\mathtt{GQ}(\lambda)$ \cite{maei2010GQ}, 
$\mathtt{ABQ}(\zeta)$ \cite{Mahmood2017_b_multi},
$\mathtt{GTB}(\lambda)$ \cite{touati2018convergent}.
We choose the above three methods as baselines due to they are all learning by expected TD-error $\delta_{t}^{\text{ES}}$, which is the same as $\mathtt{GES}(\lambda)$. 
For the limitation of space, we present some details of the experiments in Appendix J.

\subsection{The Effect of Step-size} In this section, we verify the convergence result presented in Theorem \ref{theo:on-algo2-convergence}/Remark \ref{remark-4}.
We use 
the empirical \[\text{MSPBE}=\frac{1}{2}\|\hat{b}-\hat{A}\theta\|^{2}_{\hat{M}^{-1}}\] 
to evaluate the performance of all the algorithms, where we evaluate $\hat{A}$, $\hat{b}$, and $\hat{M}$ according to their unbiased estimates by Monte Carlo method with 5000 episodes.
Particular, for Mountaincar,
to collect the samples, we run $\mathtt{Sarsa}$ with $p=128$ features to obtain a stable policy. 
Then, we use this policy to collect trajectories that comprise the samples.

Figure 3 shows the comparison of the empirical MSPBE performance between a constant step-size and the decay step-size $\frac{1}{\sqrt{t}}$. 
Result (in Figure 3) illustrates that the $\mathtt{GES}(\lambda)$with a proper constant step-size converges significantly faster than the learning with step-size $\frac{1}{\sqrt{t}}$, which support our theory analysis in Remark \ref{remark-4}.

\subsection{Comparison of Empirical MSPBE}
The MSPBE distribution is computed over the combination of step-size,
$(\alpha_{k},\frac{\beta_{k}}{\alpha_{k}})\in[0.1\times 2^{j}|j = -10,-9,\cdots,-1, 0]^{2}$, and we set $\lambda=0.99$, $\zeta=0.95$ for $\mathtt{ABQ}(\zeta)$..
All the result showed in Figure 4 is an average of 100 runs.

Result in Figure 4 shows that our $\mathtt{GES}(\lambda)$ learns significantly faster with better performance than $\mathtt{GQ}(\lambda)$, $\mathtt{ABQ}(\zeta)$ and $\mathtt{GTB}(\lambda)$ in all domains. 
Besides, $\mathtt{GES}(\lambda)$ converges with a lower variance.
We also notice that although Touati et al\shortcite{touati2018convergent} claim their $\mathtt{GTB}(\lambda)$ reaches the same convergence rate as our $\mathtt{GES}(\lambda)$,
result in Figure 4 shows that our $\mathtt{GES}(\lambda)$ outperforms their $\mathtt{GTB}(\lambda)$ siginificantly.

\subsection{Comparison of Empirical MSE}
We use the following empirical $\text{MSE}$ according to \cite{adam2016investigating},
\[
\text{MSE}=\|\Phi\theta-q^{\pi}\|_{\Xi}
,\] where $q^{\pi}$ is estimated by simulating the target policy and
averaging the discounted cumulative rewards overs trajectories.
The combination of step-size for MSE is the same as previous empirical MSPBE.
All the result showed in Figure 5 is an average of 100 runs and we set $\lambda=0.99$, $\zeta=0.95$ for $\mathtt{ABQ}(\zeta)$.

The result in Figure 5 shows that $\mathtt{GES}(\lambda)$ converges significantly faster than all the three baselines with lower variance in Mountaincar domain. 
For the Two-state MDP and Baird domain, $\mathtt{GES}(\lambda)$ also achieves a better performance.
This conclusion further verifies the effectiveness of the proposed $\mathtt{GES}(\lambda)$.

\section{Conclusion}
In this paper, we introduce control variate technique to $\mathtt{Expected~Sarsa}$($\lambda$) and propose $\mathtt{ES}(\lambda)\text{-}\mathtt{CV}$ algorithm.
We analyze all the random sources lead to the variance of $\mathtt{ES}(\lambda)\text{-}\mathtt{CV}$. 
We prove that if a good estimator of value function achieves, the $\mathtt{ES}(\lambda)\text{-}\mathtt{CV}$ enjoys a lower variance than Expected Sarsa($\lambda$) without control variate.
Then, we extend $\mathtt{ES}(\lambda)\text{-}\mathtt{CV}$ to be a convergent algorithm with function approximation and propose $\mathtt{GES}(\lambda)$ algorithm.
We prove that the convergence rate of $\mathtt{GES}(\lambda)$ achieves $\mathcal{O}(1/T)$, which matches or outperforms several state-of-art gradient-based algorithms, but we use a more relaxed step-size.
Finally, we use numerical experiments to demonstrate the effectiveness of the proposed algorithm. Results show that the proposed algorithm converges faster and with lower variance than three typical algorithms $\mathtt{GQ}$($\lambda$), $\mathtt{GTB}$($\lambda$) and $\mathtt{ABQ}$($\zeta$).

\bibliographystyle{aaai}
\bibliography{reference}
\clearpage

\onecolumn
\clearpage
\appendix
\section{Appendix A: $\lambda$-Return of Sarsa for Off-policy Learning}
\label{EQ-for-rec}

For the discussion of off-policy learning, we need the background of importance sampling. Thus, the basic common conclusion about importance sampling (IS) and pre-decision importance sampling (PDIS) \cite{precup2000eligibility} is necessary.
\subsection{Off-Policy Learning via Importance Sampling}
\label{app:prop1-1}

Usually, we require that every action taken by $\pi$ is also taken by $\mu$,
which is often called \emph{coverage}~\cite{sutton2018reinforcement} in reinforcement learning.
\begin{assumption}[Coverage]
    \label{ass:Coverage}
    $\forall~(s,a)\in \mathcal{S}\times\mathcal{A}$, we require that
    $\pi(a|s) > 0 \Rightarrow \mu(a|s) > 0$.
\end{assumption}
The difficulty of off-policy roots in the discrepancy between target policy $\pi$ and behavior policy $\mu$ 
----we want to learn the target policy while we only get the data generated by behavior policy. 
One technique to hand this discrepancy is \emph{importance sampling} (IS)~\cite{rubinstein2016simulation}.
Let $\tau_{t}^{h}=\{S_{t},A_{t},R_{t+1}\}_{t\ge0}^{h}$ be a trajectory with finite horizon $h<\infty$.
Let $\rho_{t:k}=\prod_{i=t}^{k}\rho_{i}$ denote the
\emph{cumulated importance sampling ratio}, 
where $\rho_{i}=\frac{\pi(A_{i}|S_{i})}{\mu(A_{i}|S_{i})}$ and $k\leq h$.
Let $G_{t}^{h}=\sum_{k=0}^{h-t-1}\gamma^{k}R_{k+t+1}$, under Assumption~\ref{ass:Coverage} the IS estimator $G_{t}^{\text{IS}}=\rho_{t:h-1}G^{h}_{t}$ is a unbiased estimation of $q^{\pi}$.
However, it is known that IS estimator suffers from large variance of the product $\rho_{t:h-1}$ \cite{sutton1998reinforcement}.
Pre-decision importance sampling (PDIS) \cite{precup2000eligibility} $G_{t}^{\text{PDIS}}=\sum_{k=0}^{h-t-1}\gamma^{k}\rho_{t:t+k}R_{t+k+1}$ is a practical variance reduction method without introducing bias, i.e.
$\mathbb{E}_{\mu}[G_{t}^{\text{PDIS}}|S_{t}=s,A_{t}=a]=q^{\pi}(s,a)$.

\begin{flalign}
\nonumber
\mathbb{E}_{\mu}[\rho_{t:h-1}G_{t}^{h}]
=&
\mathbb{E}_{\mu}[\underbrace{\rho_{t:h-1}R_{t+1}+\rho_{t:h-1}\gamma R_{t+2}+\cdots+\rho_{t:h-1}\gamma^{h-t-1}R_{h}}_{\overset{\text{def}}=G_{t}^{\text{IS}}~~\text{IS-return}}]\\
\nonumber
=&\mathbb{E}_{\mu}[\underbrace{\rho_{t}R_{t+1}+\rho_{t:t+1}\gamma R_{t+2}+\cdots+\rho_{t:h-1}\gamma^{h-t-1}R_{h}}_{\overset{\text{def}}=G_{t}^{\text{PDIS}}~~\text{PDIS-return}}]
=\mathbb{E}_{\mu}[\sum_{k=0}^{h-t-1}\gamma^{k}\rho_{t:t+k}R_{t+k+1}].
\end{flalign}
For the equation $\mathbb{E}_{\mu}[G_{t}^{\text{IS}}]=\mathbb{E}_{\mu}[G_{t}^{\text{PDIS}}]$, please see\cite{precup2000eligibility} or section 5.9 in \cite{sutton2018reinforcement}.
\begin{lemma}[Section 3.10, \cite{thomas2015safe}; Section 5.9,~\cite{sutton2018reinforcement}]
\label{app_A_lemma1}
Let $\tau_{t}^{h}=\{S_{k},A_{k},R_{k+1}\}_{k=t}^{h}$ be the trajectory generated by behavior policy $\mu$, for a given policy $\pi$ and under Assumption~\ref{ass:Coverage}, the following holds,
\begin{flalign}
\label{PDIS-Reward}
\mathbb{E}_{\mu}[\rho_{t:h-1}R_{t+k}]=\mathbb{E}_{\mu}[\rho_{t:t+k-1}R_{t+k}].
\end{flalign}
\end{lemma}
Lemma \ref{app_A_lemma1} implies that for any time $t+k~(k\ge0)$, 
the importance sampling factors after $t+k$ have no effect in the expectation, thus the following holds: for all $k\ge0$,
\begin{flalign}
\label{app1_lemma_corr}
\mathbb{E}_{\mu}[\rho_{t:h-1}R_{t+k}]=\mathbb{E}_{\mu}[\rho_{t:t+k-1}R_{t+k}]=\mathbb{E}_{\pi}[R_{t+k}].
\end{flalign}

\subsection{$\lambda$-Return of Sarsa}
\label{app:prop1-2}

The \emph{$\lambda$-return}~\cite{sutton1998reinforcement} is 
an average contains all the \emph{$n$-step return} by weighting proportionally to $\lambda^{n-1}$, $\lambda\in[0,1]$. 
For example, let $G_{t}^{t+n}=\sum_{i=0}^{n-1}\gamma^{i}R_{t+i+1}+\gamma^{n}Q_{t+n}$ be $n$-step return, then the \emph{standard forward view} of Sarsa$(\lambda)$ is
$
G_{t}^{\lambda,\text{S}}=(1-\lambda)\sum_{n=1}^{\infty}\lambda^{n-1}G_{t}^{t+n}
$, 
which is equivalent to the following recursive version
\begin{flalign}
\nonumber
G_{t}^{\lambda,\text{S}}=R_{t+1}+\gamma[(1-\lambda)Q_{t+1}+\lambda G_{t+1}^{\lambda,\text{S}}].
\end{flalign}

We only discuss the case of off-policy learning. On-Policy is a particular case of off-policy learning if $\rho_{t}=1$.
One version of $\lambda$-return of off-policy Sarsa$(\lambda)$
via importance sampling is defined as the following recursive iteration (Section 12.8,~\cite{sutton2018reinforcement}):
\begin{flalign}
\label{l_return_offpolicysarsa-1}
G_{t}^{\lambda\rho,\text{S}}=\rho_{t}(R_{t+1}+\gamma[(1-\lambda)Q_{t+1}+\lambda G_{t+1}^{\lambda\rho,\text{S}}]).
\end{flalign}
The next Proposition \ref{prop1} gives a forward view of Eq.(\ref{l_return_offpolicysarsa-1}),
and $G_{t}^{\lambda\rho,\text{S}}$ is an unbiased estimate of $q^{\pi}$.
\begin{proposition}
    \label{prop1}
    Let $\mu$ be behavior policy and $\pi$ be the target policy. 
    $G_{t}^{t}=Q_{t}$, $G_{t}^{t+n}=\rho_{t}(R_{t+1}+\gamma G_{t+1}^{t+n})$, and
    $G_{t}^{\lambda\rho}=(1-\lambda)\sum_{n=1}^{\infty}\lambda^{n-1}G_{t}^{t+n}$,
    then $G_{t}^{\lambda\rho}$ is equivalent to $G_{t}^{\lambda\rho,\emph{S}}$ defined in Eq.(\ref{l_return_offpolicysarsa-1}).
    Furthermore, $\mathbb{E}_{\mu}[G_{t}^{\lambda\rho}|(S_{t},A_{t})=(s,a)]=q^{\pi}(s,a)$.
\end{proposition}

\begin{proof}
We restate the complete calculation process of off-policy $\lambda$-return $G_{t}^{\lambda\rho}$ as belowing
\begin{flalign}
\label{app1:Gttn}
G_{t}^{t}&=Q_{t},G_{t}^{t+n}=\rho_{t}(R_{t+1}+\gamma G_{t+1}^{t+n}),\\
\label{app1:Gtlambda_rho}
G_{t}^{\lambda\rho}&=(1-\lambda)\sum_{n=1}^{\infty}\lambda^{n-1}G_{t}^{t+n}\\
\nonumber
&=(1-\lambda)G_{t}^{t+1}+\lambda(1-\lambda)\sum_{n=1}^{\infty}\lambda^{n-1}G_{t}^{t+n+1}\\
\nonumber
&=(1-\lambda)\rho_{t}(R_{t+1}+\gamma Q_{t+1})+\lambda(1-\lambda)\sum_{n=1}^{\infty}\lambda^{n-1}
\Big(
\underbrace{\rho_{t}(R_{t+1}+\gamma G_{t+1}^{t+n+1})}_{=G_{t}^{t+n+1};\text{Eq.(\ref{app1:Gttn})}|_{n\leftarrow n+1}}
\Big)\\
\nonumber
&=(1-\lambda)\rho_{t}(R_{t+1}+\gamma Q_{t+1})+\lambda\rho_{t}R_{t+1}
+\gamma\lambda\Big[\underbrace{(1-\lambda)\sum_{n=1}^{\infty}\lambda^{n-1}G_{t+1}^{t+n+1}}_{=G_{t+1}^{\lambda\rho};\text{Eq.(\ref{app1:Gtlambda_rho})}|_{t\leftarrow t+1}}\Big]\\
\label{app1_gtlamda_expend}
&=\rho_{t}\Big(R_{t+1}+\gamma[(1-\lambda)Q_{t+1}+\lambda G_{t+1}^{\lambda\rho}]\Big).
\end{flalign}
The last Eq.(\ref{app1_gtlamda_expend}) implies that from the definition of standard $\lambda$-return Eq.(\ref{app1:Gttn}) and Eq.(\ref{app1:Gtlambda_rho}), we can get the recursive form of Eq.(\ref{l_return_offpolicysarsa-1}).
 
Expanding Eq.(\ref{app1:Gttn}), we get the complete $n$-step return as follows
\begin{flalign}
G_{t}^{t+n}=\sum_{k=1}^{n}\gamma^{k-1}\rho_{t:t+k-1}R_{t+k}+\gamma^{n}\rho_{t:t+n}Q(S_{t+n},A_{t+n}).
\end{flalign}
By Eq.(\ref{PDIS-Reward}) and Eq.(\ref{app1_lemma_corr}), we have 
\begin{flalign}
\nonumber
&~~~~~\mathbb{E}_{\mu}[G_{t}^{t+n}|(S_{t},A_{t})=(s,a)]\\
\nonumber
&=\mathbb{E}_{\mu}[\sum_{k=1}^{n}\gamma^{k-1}\rho_{t:t+k-1}R_{t+k}+\gamma^{n}\rho_{t:t+n}Q(S_{t+n},A_{t+n})|(S_{t},A_{t})=(s,a)]\\
\label{app-1}
&=\mathbb{E}_{\pi}[\sum_{k=1}^{n}\gamma^{k-1}R_{t+k}+\gamma^{n}Q(S_{t+n},A_{t+n})|(S_{t},A_{t})=(s,a)]=q^{\pi}(s,a),
\end{flalign}
thus,
$
\mathbb{E}_{\mu}[G_{t}^{\lambda\rho}|(S_{t},A_{t})=(s,a)]
=\mathbb{E}_{\mu}[(1-\lambda)\sum_{n=1}^{\infty}\lambda^{n-1}G_{t}^{t+n}|(S_{t},A_{t})=(s,a)]
=q^{\pi}(s,a).
$
\end{proof}

\section{ Appendix B: Proof of Eq.(\ref{on-plicy-ES-lambda-recursive}) and  Proposition \ref{prop2}}

\label{app-pro-1}
\subsection{Eq.(\ref{on-plicy-ES-lambda-recursive}): Recursive $\lambda$-Return of Expected Sarsa for On-policy Case}
\label{app:ES-lambda-on-policy}

In this section, we prove \textbf{(I)} the forward view of  Eq.(\ref{on-plicy-ES-lambda-recursive}); 
\textbf{(II)} Eq.(\ref{on-plicy-ES-lambda-recursive}) is 
an unbiased estimate of $q^{\pi}$.

\emph{
Let
$
G_{t}^{\lambda,\emph{ES}}=(1-\lambda)\sum_{n=1}^{\infty}\lambda^{n-1}G_{t}^{t+n},
$
where $G_{t}^{t+n}=\sum_{i=0}^{n-1}\gamma^{i}R_{t+i+1}+\gamma^{n}\bar{Q}_{t+n}$ 
is $n$-step return of Expected Sarsa
and $\bar{Q}_{t+n}=\mathbb{E}_{\pi}[Q(S_{t+n},\cdot)]$, then
$G_{t}^{\lambda,\emph{ES}}$ can be written  recursively as: 
$
G_{t}^{\lambda,\emph{ES}}=R_{t+1}+\gamma[(1-\lambda)\bar{Q}_{t+1}+\lambda G_{t+1}^{\lambda,\emph{ES}}].
$
Besides,
$
\mathbb{E}_{\pi}[G_{t}^{\lambda,\emph{ES}}|(S_{t},A_{t})=(s,a))]=q^{\pi}(s,a).
$
}
\begin{proof}
By the definition of $n$-step return of Expected Sarsa: $G_{t}^{t+n}=\sum_{i=0}^{n-1}\gamma^{i}R_{t+i+1}+\gamma^{n}\bar{Q}_{t+n}$,
then $G_{t}^{t+n}$ can be written as  the following recursive form:
\begin{flalign}
\label{n-step-recursive-on-policy}
G_{t}^{t+n+1}=R_{t+1}+\gamma G_{t+1}^{t+n+1}.
\end{flalign}
Now, we turn to analyses $G_{t}^{\lambda,\text{ES}}$:
\begin{flalign}
\nonumber
G_{t}^{\lambda,\text{ES}}=&(1-\lambda)\sum_{n=1}^{\infty}\lambda^{n-1}G_{t}^{t+n}\\
\nonumber
=&(1-\lambda)G_{t}^{t+1}+(1-\lambda)\sum_{n=2}^{\infty}\lambda^{n-1}G_{t}^{t+n}\\
\nonumber
=&(1-\lambda)(R_{t+1}+\gamma\bar{Q}_{t+1})+\lambda(1-\lambda)\sum_{n=1}^{\infty}\lambda^{n-1}G_{t}^{t+n+1}\\
\nonumber
\overset{\text{Eq.(\ref{n-step-recursive-on-policy})}}=&(1-\lambda)(R_{t+1}+\gamma\bar{Q}_{t+1})+\lambda(1-\lambda)\sum_{n=1}^{\infty}\lambda^{n-1}[R_{t+1}+\gamma G_{t+1}^{t+n+1}]\\
\nonumber
=&(1-\lambda)(R_{t+1}+\gamma\bar{Q}_{t+1})+\lambda R_{t+1}+\gamma\lambda
\underbrace{\big[(1-\lambda)\sum_{n=1}^{\infty}\lambda^{n-1} G_{t+1}^{t+n+1}\big]}_{=G_{t+1}^{\lambda,\text{ES}}}\\
\nonumber
=&R_{t+1}+\gamma[(1-\lambda)\bar{Q}_{t+1}+\lambda G_{t+1}^{\lambda,\text{ES}}],
\end{flalign}
which is the result in Eq.(\ref{on-plicy-ES-lambda-recursive}).

For on-policy learning, the following is obvious
\begin{flalign}
\label{app-2-1}
	\mathbb{E}_{\pi}[G_{t}^{t+n}]=\mathbb{E}_{\pi}[\sum_{i=0}^{n-1}\gamma^{i}R_{t+i+1}+\gamma^{n}\bar{Q}_{t+n}]=\mathbb{E}_{\pi}[\sum_{i=0}^{n-1}\gamma^{i}R_{t+i+1}+\gamma^{n}{Q}_{t+n}].
\end{flalign}
It is similar to the Eq.(\ref{app-1}), we have
\begin{flalign}
	\mathbb{E}_{\pi}[G_{t}^{\lambda,\text{ES}}|(S_{t},A_{t})=(s,a)]&=\mathbb{E}_{\pi}[(1-\lambda)\sum_{n=1}^{\infty}\lambda^{n-1}G_{t}^{t+n}|(S_{t},A_{t})=(s,a)]\\
	&\overset{(\ref{app-2-1})}=\mathbb{E}_{\pi}\Big[(1-\lambda)\sum_{n=1}^{\infty}\lambda^{n-1}(\sum_{i=0}^{n-1}\gamma^{i}R_{t+i+1}+\gamma^{n}{Q}_{t+n})|(S_{t},A_{t})=(s,a)\Big]
	\\
	&\overset{(\ref{app-1})}=q^{\pi}(s,a),
\end{flalign}
which implies $G_{t}^{\lambda,\text{ES}}$ is an unbiased estimate of $q^{\pi}$.
\end{proof}

\subsection{Proof of Proposition \ref{prop2}}
\label{app:prop2}

\textbf{Proposition} \ref{prop2}
~~\emph
{
	Let $\mu$ and $\pi$ be the behavior and target policy, respectively. 
	Consider the $\lambda$-return of Sarsa and Eq.(\ref{off-es-recursive}), 
	then
	$
	\mathbb{E}_{\mu}[G_{t}^{\lambda\rho,\emph{ES}}|(S_{t},A_{t})=(s,a)]=\mathbb{E}_{\pi}[G_{t}^{\lambda,\emph{S}}|(S_{t},A_{t})=(s,a)]=q^{\pi}(s,a).
	$
}
\begin{proof}
	We expand $\mathbb{E}_{\mu}[G_{t}^{\lambda\rho,\text{ES}}|(S_{t},A_{t})=(s,a)]$ as follows
	\begin{flalign}
	\nonumber
	&\mathbb{E}_{\mu}[G_{t}^{\lambda\rho,\text{ES}}|(S_{t},A_{t})=(s,a)]\\
	\nonumber
	=&\mathbb{E}_{\mu}\Big[R_{t+1}+\gamma[(1-\lambda)\bar{Q}_{t+1}+\lambda\rho_{t+1}G_{t+1}^{\lambda\rho,\text{ES}}]|(S_{t},A_{t})=(s,a)\Big]\\
	\label{app-B2-1}
	=&\mathbb{E}_{\pi}\Big[R_{t+1}+\gamma[(1-\lambda)Q_{t+1}]|(S_{t},A_{t})=(s,a)\Big]
	+\mathbb{E}_{\mu}\Big[\gamma\lambda\rho_{t+1}G_{t+1}^{\lambda\rho,\text{ES}}|(S_{t+1},A_{t+1})=(s^{'},a^{'})\Big]\\
	\nonumber
	=&\mathbb{E}_{\pi}\Big[R_{t+1}+\gamma[(1-\lambda)Q_{t+1}]|(S_{t},A_{t})=(s,a)\Big]\\
	\nonumber
	&~~~~~~~~~~~~~~+\gamma\lambda\sum_{s^{'}\in\mathcal{S}}P^{a}_{ss^{'}}\sum_{a^{'}\in\mathcal{A}}\mu(a^{'}|s^{'})\frac{\pi(a^{'}|s^{'})}{\mu(a^{'}|s^{'})}
	\mathbb{E}_{\mu}[G_{t+1}^{\lambda\rho,\text{ES}}|(S_{t+1},A_{t+1})=(s^{'},a^{'})]\\
		\label{app-B2-2}
	=&\mathbb{E}_{\pi}\Big[R_{t+1}+\gamma(1-\lambda)Q_{t+1}+\gamma\lambda\mathbb{E}_{\mu}[G_{t+1}^{\lambda\rho,\text{ES}}|(S_{t+1},A_{t+1})=(s^{'},a^{'})]\Big|(S_{t},A_{t})=(s,a)\Big],
	\end{flalign}
	where Eq.(\ref{app-B2-1}) holds by the following facts: recall $\bar{Q}_{t+1}=\sum_{a\in\mathcal{A}}\pi(a|S_{t+1})Q_{t+1}(S_{t+1},a)$, thus
	\[\mathbb{E}_{\mu}[\bar{Q}_{t+1}]=\sum_{a\in\mathcal{A}}\mu(a|S_{t+1})\bar{Q}_{t+1}=\bar{Q}_{t+1}\underbrace{\sum_{a\in\mathcal{A}}\mu(a|S_{t+1})}_{=1}=\mathbb{E}_{\pi}[Q_{t+1}].\]
	If we continue to expand Eq.(\ref{app-B2-2}), then we have
	\[
	\mathbb{E}_{\mu}[G_{t}^{\lambda\rho,\text{ES}}|(S_{t},A_{t})=(s,a)]
	=\mathbb{E}_{\pi}[G_{t}^{\lambda,\text{S}}|(S_{t},A_{t})=(s,a)]=q^{\pi}(s,a).
	\]
\end{proof}
\section{Appendix C: Proof of  Theorem \ref{ES-Sarsa-CV-bias-variance}}
\label{app:control_variate}

\textbf{Theorem} \ref{ES-Sarsa-CV-bias-variance} 
(Forward View and Variance Analysis of Expected Sarsa$(\lambda)$ with Control Variate)
\emph{
Let $\mu$ and $\pi$ denote the behavior and target policy, respectively. 
The $\lambda$-return with control variate defined in Eq.(\ref{es-recursive-cv}) is equivalent to the following forward view: let $G_{t}^{t}=Q_{t}$,
\begin{flalign}
\label{app-2}
G_{t}^{t+n}&=R_{t+1} +\gamma\bar{Q}_{t+1}+\gamma(\rho_{t+1}G_{t+1}^{t+n}-\rho_{t+1}Q_{t+1}),\\
\label{app-3}
\widetilde{G}_{t}^{\lambda\rho,\emph{ES}}&=(1-\lambda)\sum_{n=1}^{\infty}\lambda^{n-1}G_{t}^{t+n}.
\end{flalign}
}
 \begin{proof} Firstly, we prove Eq.(\ref{app-2}),(\ref{app-3}) is equivalent to Eq.(\ref{es-recursive-cv}).
Let's expand $\widetilde{G}_{t}^{\lambda\rho,\text{ES}}$ (in Eq.(\ref{app-3})),
\begin{flalign}
\widetilde{G}_{t}^{\lambda\rho,\text{ES}}
=&(1-\lambda)G_{t}^{t+1}+(1-\lambda)\sum_{n=2}^{\infty}\lambda^{n-1}G_{t}^{t+n}\\
		\nonumber
		=&(1-\lambda)(\underbrace{R_{t+1}+\gamma\bar{Q}_{t+1}}_{=G_{t}^{t+1}; \text{Eq}.(\ref{app-2}),n=1} )+(1-\lambda)\lambda\sum_{n=1}^{\infty}\lambda^{n-1}G_{t}^{t+n+1}\\
		\nonumber
		=&(1-\lambda)(R_{t+1}+\gamma\bar{Q}_{t+1} )\\
		\nonumber
		&~~~~~~+(1-\lambda)\lambda\sum_{n=1}^{\infty}\lambda^{n-1}\Big(
		\underbrace{R_{t+1} +\gamma(\rho_{t+1}G_{t+1}^{t+n+1}+\bar{Q}_{t+1}-\rho_{t+1}Q_{t+1})}_{=G_{t}^{t+n+1};\text{Eq}.(\ref{app-2}),n\leftarrow n+1}
		\Big)\\
		\nonumber
		=&(1-\lambda)(R_{t+1}+\gamma\bar{Q}_{t+1} ) \\
			\nonumber
		&~~~~~+ \lambda(R_{t+1}+\gamma\bar{Q}_{t+1}-\gamma\rho_{t+1}Q_{t+1}))+\gamma\lambda\rho_{t+1}
		\underbrace{(1-\lambda)\sum_{n=1}^{\infty}\lambda^{n-1}G_{t+1}^{t+n+1}}_{=\widetilde{G}_{t+1}^{\lambda\rho,\text{ES}}; \text{Eq}.(\ref{app-3})t\leftarrow t+1}\\
		\label{app-4}
		&=R_{t+1}+\gamma\Big(\bar{Q}_{t+1}+\lambda\rho_{t+1}(\widetilde{G}_{t+1}^{\lambda\rho,\text{ES}}-Q_{t+1})\Big),
\end{flalign}
the last Eq.(\ref{app-4}) implies
\begin{flalign}
\label{app-5}
\widetilde{G}_{t}^{\lambda\rho,\text{ES}}=R_{t+1}+\gamma\big[
(1-\lambda)\bar{Q}_{t+1}+\lambda\big(\rho_{t+1}\widetilde{G}_{t+1}^{\lambda\rho,\text{ES}}
+\bar{Q}_{t+1}-\rho_{t+1}Q_{t+1}\big)
\big],
\end{flalign}
which is the 
Eq.(\ref{es-recursive-cv})
\end{proof}

\section{Appendix D: Proof of Eq.(\ref{operator-es})}

\subsection{The Equivalence (a) for Eq.(\ref{operator-es}) }
\label{app-(a)}
\begin{proof}
    \begin{flalign}
    \nonumber
        q+\mathbb{E}_{\mu}[\sum_{l=t}^{\infty}(\lambda\gamma)^{l-t}\delta^{\text{ES}}_{l}\rho_{t+1:l}]\overset{(\ref{app1_lemma_corr})}=&q+\mathbb{E}_{\pi}[\sum_{l=t}^{\infty}(\lambda\gamma)^{l-t}\delta^{\text{ES}}_{l}]\\
        \label{app-7}
        =&q+(I-\lambda\gamma P^{\pi})^{-1}(\mathcal{B}^{\pi}q-q),
    \end{flalign}
    Eq. (\ref{app-7}) is a common result in RL, the details of $\mathbb{E}_{\mu}[\sum_{l=t}^{\infty}(\lambda\gamma)^{l-t}\delta^{\text{ES}}_{l}\rho_{t+1:l}]=\mathbb{E}_{\pi}[\sum_{l=t}^{\infty}(\lambda\gamma)^{l-t}\delta^{\text{ES}}_{l}]$ please refer to  \cite{geist2014off} or Section 6.3.9 in \cite{bertsekas2012dynamic}.
\end{proof}

\section{Appendix E: Proof of Theorem \ref{theorem-ope}}
\label{proof-theorem-ope}

\textbf{Theorem} \ref{theorem-ope} (Policy Evaluation)
\emph{
    For any initial $Q_{0}$, consider the sequential trajectory collection $\mathcal{T}$, and the following $Q_{k}$ is learned according to the $k$-th trajectory $\tau_{k}$, $k\ge1$,
    \begin{flalign}
    \nonumber
    Q_{k+1}=\mathcal{B}^{\pi}_{\lambda}Q_{k}.
    \end{flalign}
    By iterating over $k$ trajectories, the error of policy evaluation is upper bounded by 
    \begin{flalign}
    \nonumber
    \|Q_{k}-q^{\pi}\|\leq\big(\frac{\gamma-\lambda\gamma}{1-\lambda\gamma}\big)^{k}\|Q_0-q^{\pi}\|.
    \end{flalign}
}
\begin{proof}(Proof of Theorem \ref{theorem-ope})
By Eq.(\ref{operator-es}), the following equation holds~\cite{geist2014off,bertsekas2017abstract},
\begin{flalign}
\label{app-bell-1}
\mathcal{B}_{\lambda}^{\pi}=(1-\lambda)\sum_{n=0}^{\infty}\lambda^{n}(\mathcal{B}^{\pi})^{n+1}.
\end{flalign}
It is known that Bellman operator $\mathcal{B}^{\pi}$ is a $\gamma$-contraction
\cite{bertsekas2017abstract},  \[\|\mathcal{B}^{\pi}Q_1-\mathcal{B}^{\pi}Q_2\|\leq\gamma\|Q_1-Q_2\|.\]
Thus we have
\begin{flalign}
    \nonumber
    \|\mathcal{B}^{\pi}_{\lambda}Q_1-\mathcal{B}^{\pi}_{\lambda}Q_2\|&\overset{(\ref{app-bell-1})}\leq(1-\lambda)\sum_{n=0}^{\infty}\lambda^{n}\|(\mathcal{B}^{\pi})^{n+1}(Q_{1}-Q_{2})\|\\
    \nonumber
    &\leq(1-\lambda)\sum_{n=0}^{\infty}\lambda^{n}\gamma\|(\mathcal{B}^{\pi})^{n}(Q_{1}-Q_{2})\|\\
     \nonumber
     &\cdots
     \\
     \nonumber
     &\leq(1-\lambda)\sum_{n=0}^{\infty}\lambda^{n}\gamma^{n+1}\|Q_{1}-Q_{2}\|\\
     \label{app-6}
     &=\dfrac{(1-\lambda)\gamma}{1-\lambda\gamma}\|Q_{1}-Q_{2}\|.
\end{flalign}
Since $0<\dfrac{(1-\lambda)\gamma}{1-\lambda\gamma}<1$, Eq.(\ref{app-6})
implies that $\mathcal{B}^{\pi}_{\lambda}$ is a $\dfrac{(1-\lambda)\gamma}{1-\lambda\gamma}$-contraction.
By Banach fixed point theorem \cite{conway2013course}, $\{Q_{k}\}_{k\ge0}$ generated by $Q_{k+1}=\mathcal{B}^{\pi}_{\lambda}Q_{k}$
converges to the fixed point of $\mathcal{B}^{\pi}_{\lambda}$.

By Eq.(\ref{operator-es}), $q^{\pi}$ is the unique fixed point of $\mathcal{B}^{\pi}_{\lambda}$.
Thus, $Q_{k+1}$ converges to $q^{\pi}$.

Now, we turn to consider the convergence rate. According to (\ref{app-6}), it is easy to see $\forall k \in\mathbb{N}$,
$
\|Q_{k+1}-Q_{k}\|\leq \dfrac{(1-\lambda)\gamma}{1-\lambda\gamma}\|Q_{k}-Q_{k-1}\|
.$
Then, $\forall k,n\in\mathbb{N}$,
\begin{flalign}
\nonumber
    \|Q_{k+n}-Q_{k}\|&\leq\dfrac{(1-\lambda)\gamma}{1-\lambda\gamma}\|Q_{k+n-1}-Q_{k-1}\|\\
    \nonumber
    &\leq(\dfrac{(1-\lambda)\gamma}{1-\lambda\gamma})^{2}\|Q_{k+n-2}-Q_{k-2}\|\\
    \nonumber
    &\cdots\\
    \nonumber
    &\leq(\dfrac{(1-\lambda)\gamma}{1-\lambda\gamma})^{k}\|Q_{n}-Q_{0}\|,
\end{flalign}
let $n\rightarrow\infty$, we have 
\begin{flalign}
\nonumber
\|Q_{k}-q^{\pi}\|\leq\big(\frac{\gamma-\lambda\gamma}{1-\lambda\gamma}\big)^{k}\|Q_0-q^{\pi}\|.
\end{flalign}
\end{proof}

\section{Appendix F: Proof of Theorem \ref{Variance-Analysis}}

\textbf{Theorem \ref{Variance-Analysis}}
\emph{
    $\widetilde{G}_{t}^{\lambda\rho,\emph{ES}}$ is an unbiased estimator of $q^{\pi}$, whose variance is given recursively as follows, 
    \begin{flalign}
    \nonumber
    \mathbb{V}{\emph{ar}}\big[\widetilde{G}_{t}^{\lambda\rho,\emph{ES}}\big]=
    &\mathbb{V}{\emph{ar}}\big[R_{t+1}+\gamma\bar{Q}_{t+1}-q^{\pi}(s,a)\big]
    +\gamma^{2}\lambda^{2}\mathbb{V}{\emph{ar}}\big[
    v^{\pi}(s^{'})-\bar{Q}_{t+1}\big]
    \\
    \nonumber
    &~~~+\gamma^{2}\lambda^{2}\mathbb{V}{\emph{ar}}[\Delta_{t+1}]
    +\gamma^{2}\lambda^{2}\mathbb{V}{\emph{ar}}\big[\rho_{t+1}\widetilde{G}_{t+1}^{\lambda\rho,\emph{ES}}\big],
    \end{flalign}
    where $t\ge0$, $\Delta_{t+1}=\bar{Q}_{t+1}-\rho_{t+1}Q_{t+1}-v^{\pi}(s^{'})+\rho_{t+1}q^{\pi}(s^{'},a^{'})$.
}

\begin{lemma}
    \label{app-var}
    The expectation of the cross-term between the TD error at $t$ and the difference between the return and value at $t+1$ is zero: 
    for any $q(s,a) = \mathbb{E}[G_{t+1}|S_t=s,A_t=a]$, i.e., satisfying the Bellman equation, for any bounded
    function $b: \mathcal{S} \times \mathcal{A} \times \mathcal{R} \times \mathcal{S} \rightarrow \mathbb{R}$,
    \begin{flalign}
    \label{app-8}
    \mathbb{E}[b(S_t, A_t, R_{t+1}, S_{t+1})(G_{t+1}-q(S_{t+1},A_{t+1}))|S_t=s,A_t=a]=0.
    \end{flalign}
\end{lemma}
A similar result of state value function appears in \cite{sherstan2018directly}, and Lemma \ref{app-var} expends it to state-action value function. Thus,we omit its proof, and for the details 
please refer to  \cite{sherstan2018directly}.

\begin{remark}
    \label{app-remark-1}
    If $G_{t+1}$ is replaced by Expected Sarsa estimator $R_{t+1}+\gamma\bar{Q}_{t+1}$, Eq.(\ref{app-8}) holds.
\end{remark}
\begin{proof}
    \textbf{(Proof of Theorem \ref{Variance-Analysis})}
    \begin{flalign}
    \nonumber
    &\mathbb{V}\text{ar}\big[\widetilde{G}_{t}^{\lambda\rho,\text{ES}}\big]\\
    \nonumber
    =&\mathbb{E}\big[(\widetilde{G}_{t}^{\lambda\rho,\text{ES}})^{2}\big]-\underbrace{\big(\mathbb{E}\big[\widetilde{G}_{t}^{\lambda\rho,\text{ES}}\big]\big)^{2}}_{=(q^{\pi}(s,a))^{2}; \text{Proposition} \ref{prop2},(\ref{es-recursive-cv})}\\
    \nonumber
    \overset{(\ref{es-recursive-cv})}=&\mathbb{E}\Big[\Big(R_{t+1}+\gamma\big[
    (1-\lambda)\bar{Q}_{t+1}+\lambda\big(\rho_{t+1}\widetilde{G}_{t+1}^{\lambda\rho,\text{ES}}
    +\bar{Q}_{t+1}-\rho_{t+1}Q_{t+1}\big)
    \big]\Big)^{2}-(q^{\pi}(s,a))^{2}\Big]\\
    \nonumber
    =&\mathbb{E}\Big[\Big(R_{t+1}+\gamma\big[
    (1-\lambda)\bar{Q}_{t+1}
    \\
    \nonumber
    &~+\lambda\big(\rho_{t+1}\widetilde{G}_{t+1}^{\lambda\rho,\text{ES}}
    +\underbrace{\bar{Q}_{t+1}-v^{\pi}(s^{'})-\rho_{t+1}Q_{t+1}+\rho_{t+1}q^{\pi}(s^{'},a^{'})}_{\Delta_{t+1}}+v^{\pi}(s^{'})-\rho_{t+1}q^{\pi}(s^{'},a^{'})\big)
    \big]\Big)^{2}
        \\
    \nonumber
    &~~~~~~~~~~~~~~~~~~~~~~~~~~~~~~~~~~~~~~~~~~~~~~~~~~~~~~~~~~~~~~~~~~~~~~~~-(q^{\pi}(s,a))^{2}\Big]
    \\
    \nonumber
    =&\mathbb{E}\bigg[\bigg(R_{t+1}+\gamma\Big[
    (1-\lambda)\bar{Q}_{t+1}
    +\lambda\Big(\rho_{t+1}\big(\widetilde{G}_{t+1}^{\lambda\rho,\text{ES}}-q^{\pi}(s^{'},a^{'})\big)
    +\Delta_{t+1}+v^{\pi}(s^{'})\Big)
    \Big]\bigg)^{2}\\
    \nonumber
    &~~~~~~~~~~~~~~~~~~~~~~~~~~~~~~~~~~~~~~~~~~~~~~~~~~~~~~~~~~~~~~~~~~~~~~~~-(q^{\pi}(s,a))^{2}\bigg]
    \\
    \nonumber
    =&\mathbb{E}\bigg[\bigg(
    R_{t+1}+\gamma\bar{Q}_{t+1}+
    \gamma\lambda(v^{\pi}(s^{'})-\bar{Q}_{t+1})
    +\gamma\lambda\Big(\rho_{t+1}\big(\widetilde{G}_{t+1}^{\lambda\rho,\text{ES}}-q^{\pi}(s^{'},a^{'})\big)
    +\Delta_{t+1})\Big)
    \bigg)^{2}
    \\
    \nonumber
    &~~~~~~~~~~~~~~~~~~~~~~~~~~~~~~~~~~~~~~~~~~~~~~~~~~~~~~~~~~~~~~~~~~~~~~~~-(q^{\pi}(s,a))^{2}\bigg]
    \\
    \nonumber
    =&\mathbb{E}\bigg[\bigg(
    R_{t+1}+\gamma\bar{Q}_{t+1}-q^{\pi}(s,a)+
    \gamma\lambda(v^{\pi}(s^{'})-\bar{Q}_{t+1})
    \\
    \nonumber
    &~~~~~~~~~~~~~~~~~~~~~~~~~~~~~~~~~+\gamma\lambda\Big(\rho_{t+1}\big(\widetilde{G}_{t+1}^{\lambda\rho,\text{ES}}-q^{\pi}(s^{'},a^{'})\big)
    +\Delta_{t+1})\Big)+q^{\pi}(s,a)
    \bigg)^{2}-(q^{\pi}(s,a))^{2}\bigg]\\
    \nonumber
    =&\mathbb{E}\Big[\Big(    R_{t+1}+\gamma\bar{Q}_{t+1}-q^{\pi}(s,a)\Big)^{2}\Big]
    +\gamma^{2}\lambda^{2}\mathbb{E}\Big[
    \big(v^{\pi}(s^{'})-\bar{Q}_{t+1}\big)^{2}\Big]+\gamma^{2}\lambda^{2}\mathbb{E}[\Delta^{2}_{t+1}]
    \\
    \label{app-9}
    &~~~~~~~~~~~~~~~~~~~~~~~~~~~~~~~~~~~~~~~~~~~~~~~~~~~~~~~~~~~~~+\gamma^{2}\lambda^{2}\mathbb{E}\Big[\rho^{2}_{t+1}\big(\widetilde{G}_{t+1}^{\lambda\rho,\text{ES}}-q^{\pi}(s^{'},a^{'})\big)^{2}\Big]
    \end{flalign}
    Eq.(\ref{app-9}) holds due to Remark \ref{app-remark-1} and Lemma 1 in \cite{sherstan2018directly}.
    By the definition of variance, 
    Eq.(\ref{app-9}) is equivalent to Eq.(\ref{variance-1}), which is the result we want to prove.
\end{proof}

\section{Appendix G: Two-State MDP Example }
\label{app-example}
\begin{flalign}
\nonumber
P^{\pi} = \begin{pmatrix}
0 & 1 & 0 & 0\\
0 & 1 & 0 & 0\\
1 & 0 & 0 & 0\\
1 & 0 & 0 & 0
\end{pmatrix}
\Longrightarrow
(I-\gamma\lambda P^{\pi}) = \begin{pmatrix}
1 & -\gamma\lambda & 0 & 0\\
0 & 1-\gamma\lambda & 0 & 0\\
-\gamma\lambda & 0 & 1 & 0\\
-\gamma\lambda & 0 & 0 & 1
\end{pmatrix},
\end{flalign}
then, we have
\begin{flalign}
\nonumber
(I-\gamma\lambda P^{\pi})^{-1} = \begin{pmatrix}
1 & \frac{\gamma\lambda}{1-\gamma\lambda} & 0 & 0\\
0 & \frac{1}{1-\gamma\lambda}& 0 & 0\\
\gamma\lambda & \frac{\gamma^2\lambda^2}{1-\gamma\lambda} & 1 & 0\\
\gamma\lambda & \frac{\gamma^2\lambda^2}{1-\gamma\lambda} & 0 & 1
\end{pmatrix}.
\end{flalign}
\begin{flalign}
\nonumber
A&= \underbrace{\begin{pmatrix}
1 & 2& 0 & 0\\
0 & 0& 1& 2
\end{pmatrix}}_{=\Phi^{\top}}
\underbrace{\frac{1}{2}I}_{=\Xi}
\underbrace{
\begin{pmatrix}
1 & \frac{\gamma\lambda}{1-\gamma\lambda} & 0 & 0\\
0 & \frac{1}{1-\gamma\lambda}& 0 & 0\\
\gamma\lambda & \frac{\gamma^2\lambda^2}{1-\gamma\lambda} & 1 & 0\\
\gamma\lambda & \frac{\gamma^2\lambda^2}{1-\gamma\lambda} & 0 & 1
\end{pmatrix}
}_{=(I-\gamma\lambda P^{\pi})^{-1}}
\underbrace{
 \begin{pmatrix}
-1 & \gamma & 0 & 0\\
0 & \gamma-1 & 0 & 0\\
\gamma & 0 & -1 & 0\\
\gamma & 0 & 0 & -1
\end{pmatrix}
}_{=\gamma P^\pi -I}
\underbrace{
	\begin{pmatrix}
1 & 0\\
2 & 0\\
0 & 1\\
0 &  2
\end{pmatrix}
}_{=\Phi}\\
&= \begin{pmatrix}
\dfrac{6\gamma-\gamma\lambda-5}{2(1-\gamma\lambda)} & 0 \\
\dfrac{3\gamma}{2} &- \dfrac{5}{2} 
\end{pmatrix}.
\end{flalign}

\section{Appendix H: Proof of Eq.(\ref{Eq:mspbe})}
\label{app-mspbe}

For a given policy $\pi$, $Q_{\theta}=\Phi\theta$, then by the definition of MSPBE objection function, we have,
\begin{flalign}
\nonumber
\text{MSPBE}(\theta,\lambda)&=\|Q	_{\theta}-\Pi\mathcal{B}_{\lambda}^{\pi}Q	_{\theta}\|^{2}_{\Xi}\\
\nonumber
&=\|\Pi Q	_{\theta}-\Pi\mathcal{B}_{\lambda}^{\pi}Q	_{\theta}\|^{2}_{\Xi}\\
\nonumber
&=\|\Phi^{T}\Xi(Q	_{\theta}-\mathcal{B}_{\lambda}^{\pi}Q	_{\theta})\|^{2}_{({\Phi^{T}\Xi\Phi})^{-1}}\\
\nonumber
&=\|\Phi^{T}\Xi(I-\lambda\gamma P^{\pi})^{-1}(\Phi\theta-\gamma P^{\pi}\Phi\theta-R^{\pi})\|^{2}_{({\Phi^{T}\Xi\Phi})^{-1}}\\
\nonumber
&=\|\Phi^{T}\Xi(I-\lambda\gamma P^{\pi})^{-1}\big((I-\gamma P^{\pi})\Phi\theta-R^{\pi}\big)\|^{2}_{({\Phi^{T}\Xi\Phi})^{-1}}\\
&=\|b+A\theta\|^{2}_{({\Phi^{T}\Xi\Phi})^{-1}},
\end{flalign}
where $A=\Phi^{T}\Xi(I-\lambda\gamma P^{\pi})^{-1}(\gamma P^{\pi}-I)\Phi,b =\Phi\Xi (I-\lambda\gamma P^{\pi})^{-1}r.$

\section{Appendix I: Proof of Theorem \ref{theo:on-algo2-convergence}}
\label{them5}
\textbf{Theorem \ref{theo:on-algo2-convergence}}
\emph{
Consider the sequence $\{(\theta_{t},\omega_{t})\}_{t=1}^{T}$ generated by (\ref{stochastic-im}), step-size $\alpha,\beta$ are positive constants.
Let $\bar{\theta}_{T}=\frac{1}{T}(\sum_{t=1}^{T}\theta_{t})$,
$\bar{\omega}_{T}=\frac{1}{T}(\sum_{t=1}^{T}\omega_{t})$ and we chose the step-size $\alpha,\beta$ satisfy $1-\sqrt{\alpha\beta}\|A\|_{*}>0$, where $\|A\|_{*}=\sup_{\|x\|=1}\|Ax\|$ is operator norm.
If parameter $(\theta,\omega)$ is on a bounded
$D_{\theta} \times D_{\omega}$, 
i.e \emph{diam} $D_{\theta}=\sup\{\|\theta_{1}-\theta_{2}\|;\theta_{1},\theta_{2}\in D_{\theta}\}\leq\infty$, \emph{diam} $D_{\omega}$$\leq\infty$, 
$\mathbb{E}[\epsilon_{\Psi}(\bar{\theta}_{T},\bar{\omega}_{T})]$ is upper bounded by:
\begin{flalign}
\nonumber
\sup_{(\theta,\omega)}\Big\{\dfrac{1}{T}(\dfrac{\|\theta-\theta_{0}\|^{2}}{2\alpha}+\dfrac{\|\omega-\omega_{0}\|^{2}}{2\beta})\Big\}.
\end{flalign}
}
The proof of Theorem \ref{theo:on-algo2-convergence} uses a inequality (in Eq.(\ref{app3: inequality})) , we present it in the next Proposition \ref{app3:prop-inequality}.
\begin{proposition}
	\label{app3:prop-inequality}
	Consider the update of expection version in Eq.(\ref{E-al}), 	
	\[
	\omega_{t+1}=\omega_t+\beta(A\theta_t+b-M\omega_t),\theta_{t+1}=\theta_t-\alpha A^{\top}\omega_t.
	\]
	Let $F(\omega)=\frac{1}{2}\|\omega\|^{2}_{M}-b^{\top}\omega$, then for any $(\theta,\omega)\in D_{\theta}\times D_{\omega}$, the following hlods
	\begin{flalign}
	\nonumber
	&\frac{1}{2\alpha}\|\theta-\theta_{t}\|^{2}+\frac{1}{2\beta}\|\omega-\omega_{t}\|^{2}\\
	\nonumber
	\ge&\frac{1}{2\alpha}(\|\theta_{t}-\theta_{t+1}\|^{2}+\|\theta_{t+1}-\theta\|^{2})+\frac{1}{2\beta}(\|\omega_{t}-\omega_{t+1}\|^{2}+\|\omega_{t+1}-\omega\|^{2})\\
	\nonumber
	&+\Big(\langle -A\theta,\omega_{t+1}\rangle+F(\omega_{t+1})\Big)-\Big(\langle -A\theta_{t+1},\omega\rangle+F(\omega)\Big)\\
	\label{app3: inequality}
	&+\Big\langle A(\theta_{t+1}-{\theta}_{t}),\omega_{t+1}-\omega\Big\rangle.
	\end{flalign}
\end{proposition}

\begin{proof} (\textbf{Proof of Proposition \ref{app3:prop-inequality}})
	Let sub-gradients of $f$ at $x$ be denoted as $\partial f(x)$, $\partial f(x)=\{g| f(x)-f(y)\leq g^{T}(x-y) ,\forall y \in \textbf{dom}(f)\}$.
By the definition of sub-gradient , we have 
	$
	\frac{\omega_{t}-\omega_{t+1}}{\beta}+A{\theta}_{t}\in\partial F(\omega_{t+1}).
	$
	Since $F$ is convex,
	then for any$(\theta,\omega)\in D_{\theta}\times D_{\omega}$ the following holds
	\begin{flalign}
	\nonumber
	F(\omega)\ge F(\omega_{t+1})+\langle\frac{\omega_{t}-\omega_{t+1}}{\beta}+A{\theta}_{t},\omega-\omega_{t+1}
	\rangle.
	\end{flalign}
	By the \emph{law of cosines}: $2\langle a-b,c-b \rangle=\|a-b\|^{2}+\|b-c\|^{2}-\|a-c\|^{2}$, we have
	\begin{flalign}
	\nonumber
	0&=\frac{1}{2\alpha}\Big(\|\theta_{t}-\theta_{t+1}\|^{2}+\|\theta_{t+1}-\theta\|^{2}-\|\theta_{t}-\theta\|^{2}\Big)+\underbrace{\langle -A^{\top}\omega_{t+1},\theta-\theta_{t+1}\rangle}_{=-\langle A(\theta-\theta_{t+1}),\omega_{t+1}\rangle},\\
	\nonumber
	F(\omega)&\ge F(\omega_{t+1})+\frac{1}{2\beta}\Big(\|\omega_{t}-\omega_{t+1}\|^{2}+\|\omega_{t+1}-\omega\|^{2}-\|\omega_{t}-\omega\|^{2}\Big)+\langle A{\theta}_{t},\omega-\omega_{t+1}
	\rangle,
	\end{flalign}
	summing them implies the following inequality,
	\begin{flalign}
	\nonumber
	&\frac{1}{2\alpha}\|\theta-\theta_{t}\|^{2}+\frac{1}{2\beta}\|\omega-\omega_{t}\|^{2}\\
	\nonumber
	\ge&\frac{1}{2\alpha}(\|\theta_{t}-\theta_{t+1}\|^{2}+\|\theta_{t+1}-\theta\|^{2})+\frac{1}{2\beta}(\|\omega_{t}-\omega_{t+1}\|^{2}+\|\omega_{t+1}-\omega\|^{2})\\
	\nonumber
	&+\Big(\langle -A\theta,\omega_{t+1}\rangle+F(\omega_{t+1})\Big)-\Big(\langle -A\theta_{t+1},\omega\rangle+F(\omega)\Big)\\
	\nonumber
	&-\Big\langle -A(\theta_{t+1}-{\theta}_{t}),\omega_{t+1}-\omega\Big\rangle,
	\end{flalign}
	which is we want to prove.
\end{proof}
\begin{proof}(\textbf{Proof of Theorem \ref{theo:on-algo2-convergence}})  
Let $\bar{\theta}_{t}=2\theta_{t}-\theta_{t-1}$,$\epsilon=\sqrt{\frac{\beta}{\alpha}}$. 
	then for any $(\theta,\omega)\in D_{\theta}\times D_{\omega}$:
	\begin{flalign}
	\nonumber
	\Big\langle {A}(\theta_{t+1}-\bar{\theta}_{t}),\omega_{t+1}-\omega\Big\rangle=&\Big\langle A\Big((\theta_{t+1}-\theta_{t})-(\theta_{t}-\theta_{t-1})\Big),\omega_{t+1}-\omega\Big\rangle\\
	\nonumber
	=&\Big\langle A(\theta_{t+1}-\theta_{t}),\omega_{t+1}-\omega\Big\rangle-\Big\langle A(\theta_{t}-\theta_{t-1}),\omega_{t+1}-\omega_{t}\Big\rangle
	\\
	\nonumber
	&~~~
	-\Big\langle A(\theta_{t}-\theta_{t-1}),\omega_{t}-\omega\Big\rangle\\
	\nonumber
	\ge&\Big\langle A(\theta_{t+1}-\theta_{t}),\omega_{t+1}-\omega\Big\rangle
	-\|A\|_{*}\|\theta_{t}-\theta_{t-1}\|\|\omega_{t+1}-\omega_{t}\|
	\\
	\nonumber
	&~~~
	-\Big\langle A(\theta_{t}-\theta_{t-1}),\omega_{t}-\omega\Big\rangle\\
	\nonumber
	\ge&\Big\langle A(\theta_{t+1}-\theta_{t}),\omega_{t+1}-\omega\Big\rangle-\|A\|_{*}\Big(\frac{\epsilon}{2}\|\theta_{t}-\theta_{t-1}\|^{2}+\frac{1}{2\epsilon}\|\omega_{t+1}-\omega_{t}\|^{2}\Big)
	\\
	\nonumber
	&~~~-\Big\langle A(\theta_{t}-\theta_{t-1}),\omega_{t}-\omega\Big\rangle.
	\end{flalign}
	By the inequality in Proposition \ref{app3:prop-inequality}, we have
	\begin{flalign}
	\nonumber
	\frac{1}{2\alpha}\|\theta-\theta_{t}\|^{2}+\frac{1}{2\beta}\|\omega-\omega_{t}\|^{2}
	\ge&\frac{1}{2\alpha}\Big(\|\theta_{t+1}-\theta\|^{2}+
	\|\theta_{t}-\theta_{t+1}\|^{2}\Big)-\sqrt{\alpha\beta}\|A\|_{*}\frac{\|\theta_{t}-\theta_{t-1}\|^{2}}{2\alpha}\\
	\nonumber
	&+(1-\sqrt{\alpha\beta}\|A\|_{*})\frac{1}{2\beta}\|\omega_{t}-\omega_{t+1}\|^{2}+\frac{1}{2\beta}\|\omega_{t+1}-\omega\|^{2}\\
	\nonumber
	&+\Big(\langle -A\theta,\omega_{t+1}\rangle+F(\omega_{t+1})\Big)-\Big(\langle -A\theta_{t+1},\omega\rangle+F(\omega)\Big)\\
	\label{app:inequality-1}
	&+\Big\langle -A(\theta_{t}-\theta_{t-1}),\omega_{t}-\omega\Big\rangle-\Big\langle -A(\theta_{t+1}-\theta_{t}),\omega_{t+1}-\omega\Big\rangle.
	\end{flalign}
	Summing the Eq.(\ref{app:inequality-1}) from $t=0:T-1$
	\begin{flalign}
	\nonumber
	\frac{1}{2\alpha}\|\theta-\theta_{0}\|^{2}+\frac{1}{2\beta}\|\omega-\omega_{0}\|^{2}
	\ge
	&\frac{1}{2\alpha}\Big(\|\theta_{T}-\theta\|^{2}+
	\|\theta_{T}-\theta_{T-1}\|^{2}\Big)-\sqrt{\alpha\beta}\|A\|_{*}\sum_{t=1}^{T-1}\frac{\|\theta_{t}-\theta_{t-1}\|^{2}}{2\alpha}\\
	\nonumber
	&+(1-\sqrt{\alpha\beta}\|A\|_{*})\sum_{t=0}^{T-1}\frac{1}{2\beta}\|\omega_{t}-\omega_{t+1}\|^{2}+\frac{1}{2\beta}\|\omega_{T}-\omega\|^{2}\\
	\nonumber
	&+\sum_{t=0}^{T-1}\big [\big(\langle- A\theta,\omega_{t+1}\rangle+F(\omega_{t+1})\big)-\big(\langle -A\theta_{t+1},\omega\rangle+F(\omega)\big)\big]\\
	\nonumber
	&-\big\langle A(\theta_{T}-\theta_{T-1}),\omega_{T}-\omega\big\rangle.
	\end{flalign}
	By the Cauchy-Schwarz inequality ${\displaystyle |\langle \mathbf {u} ,\mathbf {v} \rangle |\leq \|\mathbf {u} \|\|\mathbf {v} \|}
	\leq\dfrac{1}{2}(\|\mathbf {u} \|^{2}+\|\mathbf {v} \|^{2})$, we have \[\Big\langle A(\theta_{T}-\theta_{T-1}),\omega_{T}-\omega\Big\rangle\leq\dfrac{1}{2\alpha}\|\theta_{T}-\theta_{T-1}\|^{2}+\alpha\beta\|A\|_{*}^{2}\dfrac{1}{2\beta}\|\omega_{T}-\omega\|^{2},\] then the following holds, for any $(\theta,\omega)\in D_{\theta}\times D_{\omega}$:
	\begin{flalign}
	\nonumber
	&\frac{1}{2\alpha}\|\theta-\theta_{0}\|^{2}+\frac{1}{2\beta}\|\omega-\omega_{0}\|^{2}\\
	\nonumber
	\ge&
	\sum_{t=0}^{T-1}\big [\big(\langle -A\theta,\omega_{t+1}\rangle+F(\omega_{t+1})\big)-\big(\langle -A\theta_{t+1},\omega\rangle+F(\omega))\big)\big]\\
	\nonumber
	&+(1-\sqrt{\alpha\beta}\|A\|_{*})\sum_{t=0}^{T-1}\frac{1}{2\beta}\|\omega_{t}-\omega_{t+1}\|^{2}+(1-\alpha\beta\|A\|_{*}^{2})\frac{1}{2\beta}\|\omega_{T}-\omega\|^{2}\\
	\label{app:inequality-2}
	&+\frac{1}{2\alpha}\|\theta_{T}-\theta\|^{2}+(1-\sqrt{\alpha\beta}\|A\|_{*})\sum_{t=1}^{T-1}\frac{\|\theta_{t}-\theta_{t-1}\|^{2}}{2\alpha}.
	\end{flalign}
	Let $\bar{\theta}_{T}=\dfrac{\sum_{t=0}^{T-1}\theta_{t}}{T}$,
	$\bar{\omega}_{T}=\dfrac{\sum_{t=0}^{T-1}\omega_{t}}{T}$ and we chose the step-size $\alpha,\beta$ satisfy $1-\sqrt{\alpha\beta}\|A\|>0$. By the convexity of $F(\omega)$ and $G(\theta)$, then we deduce from (\ref{app:inequality-2}):
	\begin{flalign}
	\label{error-bound}
	\Big(\underbrace{\langle -A\theta,\bar{\omega}_{T}\rangle+F(\bar{\omega}_{T})}_{-\Psi(\theta,\bar{\omega}_{T})}\Big)-\Big(\underbrace{\langle -A\bar{\theta}_{T},\omega\rangle+F(\omega)}_{-\Psi(\bar{\theta}_{T},\omega)}\Big)
	\leq\dfrac{1}{T}\bigg(\dfrac{\|\theta-\theta_{0}\|^{2}}{2\alpha}+\dfrac{\|\omega-\omega_{0}\|^{2}}{2\beta}\bigg).
	\end{flalign}
	By Eq.(\ref{error-bound}), we have 
	\[
	\mathbb{E}[\epsilon_{\Psi}(\bar{\theta}_{T},\bar{\omega}_{T})]\leq\sup_{(\theta,\omega)}\Bigg\{\dfrac{1}{T}\bigg(\dfrac{\|\theta-\theta_{0}\|^{2}}{2\alpha}+\dfrac{\|\omega-\omega_{0}\|^{2}}{2\beta}\bigg)\Bigg\}.
	\]
\end{proof}
\clearpage
\section{Appendix J: Details of Experiments}

\label{ex-detail}

\textbf{MountainCar}
Since the state space of mountaincar domain is continuous, we use the open tile coding software 
\url{http://incompleteideas.net/rlai.cs.ualberta.ca/RLAI/RLtoolkit/tilecoding.html} to extract feature of states.

In this experiment, we set the number of tilings to be 4 and there are no white noise features. 
The performance is an average 5 runs and each run contains 5000 episodes. 
We set $\lambda=0.99$, $\gamma=0.99$.  
The MSPBE/MSE distribution is computed over the combination of step-size,
$(\alpha_{k},\frac{\beta_{k}}{\alpha_{k}})\in[0.1\times 2^{j}|j = -10,-9,\cdots,-1, 0]^{2}$, and $\lambda=0.99$.
Following suggestions from Section10.1 in \cite{sutton2018reinforcement}, we set all the initial state-action values to be  0, which is optimistic to cause extensive exploration.

\textbf{Baird Example} The Baird example considers the episodic seven-state, two-action MDP.
The $\mathtt{dashed}$ action takes the system to one of the six upper states with equal probability, whereas the $\mathtt{solid}$ action takes the system to the seventh state. 
The behavior policy $b$ selects the $\mathtt{dashed}$ and $\mathtt{solid}$ actions with probabilities $\frac{6}{7}$ and $\frac{1}{7}$, so that the next-state distribution under it is uniform (the
same for all nonterminal states), which is also the starting distribution for each episode. The target policy $\pi$ always takes the solid action, and so the on-policy distribution (for $\pi$) is concentrated in the seventh state. The reward is zero on all transitions. The discount rate is  $\gamma=0.99$.
The feature $\phi(\cdot,{\mathtt{dashed}})$ and $\phi(\cdot,{\mathtt{solid}})$ are defined as follows,
\begin{flalign}
\nonumber
\phi(\mathtt{s_1},{\mathtt{dashed}})& = 
(2,0,0,0,0,0,0,1,0,0,0,0,,0,,0,0,0,0 )\\
\nonumber
\phi(\mathtt{s_2},{\mathtt{dashed}}) &= 
(0,2,0,0,0,0,0,1,0,0,0,0,,0,,0,0,0,0 )\\
\nonumber
&\cdots\\
\phi(\mathtt{s_7},{\mathtt{dashed}}) &= 
(0,0,0,0,0,0,2,1,0,0,0,0,,0,,0,0,0,0 ),
\end{flalign}
\begin{flalign}
\nonumber
\phi(\mathtt{s_1},{\mathtt{solid}})& = 
(0,0,0,0,0,0,0,0,2,0,0,0,,0,,0,0,0,1 )\\
\nonumber
\phi(\mathtt{s_2},{\mathtt{solid}}) &= 
(0,0,0,0,0,0,0,0,0,2,0,0,,0,,0,0,0,1 )\\
\nonumber
&\cdots\\
\phi(\mathtt{s_7},{\mathtt{solid}}) &= 
(0,0,0,0,0,0,0,0,0,0,0,0,,0,,0,0,2,1 ).
\end{flalign}

\end{document}